\documentclass{article}
\usepackage{arxiv_format/iclr_preprint,times}

\usepackage[utf8]{inputenc} 
\usepackage[T1]{fontenc}    
\usepackage{url}            
\usepackage{graphicx}
\usepackage{booktabs}       
\usepackage{amsfonts}       
\usepackage{nicefrac}       
\usepackage{microtype}      
\usepackage[table]{xcolor}         
\usepackage{amsmath}
\usepackage{mathtools,amssymb,amsthm}
\usepackage{thm-restate}
\usepackage{enumitem}
\usepackage{algorithm}
\usepackage{algpseudocode}
\usepackage{adjustbox}
\usepackage{tikz}
\usetikzlibrary{backgrounds,calc,positioning,arrows.meta,decorations.pathreplacing}
\usepackage{longtable}
\usepackage{array}
\usepackage{tabularx}
\usepackage{float}
\usepackage{multirow}
\usepackage{hyperref}       

\newcommand{\hypbox}[1]{%
  \begingroup
  \setlength{\fboxsep}{1.5pt}%
  \colorbox{hypbg}{\textcolor{hypfg}{#1}}%
  \endgroup
}

\newcommand{\anabox}[1]{%
  \begingroup
  \setlength{\fboxsep}{1.5pt}%
  \colorbox{anabg}{\textcolor{anafg}{#1}}%
  \endgroup
}

\newcommand{\solbox}[1]{%
  \begingroup
  \setlength{\fboxsep}{1.5pt}%
  \colorbox{solbg}{\textcolor{solfg}{#1}}%
  \endgroup
}

\definecolor{headergray}{RGB}{242,244,247}
\definecolor{hypbg}{RGB}{247,242,255}
\definecolor{hypfg}{RGB}{110,72,170}
\definecolor{anabg}{RGB}{240,247,255}
\definecolor{anafg}{RGB}{43,103,176}
\definecolor{solbg}{RGB}{241,250,245}
\definecolor{solfg}{RGB}{44,128,96}

\definecolor{ourfill}{HTML}{EFE7F9}
\definecolor{ourhint}{HTML}{DCCBF2}
\definecolor{ourline}{HTML}{6D4FB0}
\definecolor{ourtext}{HTML}{4A2E86}
\definecolor{bandcol}{HTML}{F6F2FC}

\newcommand{\hintbox}[1]{%
  \begingroup
  \setlength{\fboxsep}{1.5pt}%
  \colorbox{blue!10}{\textcolor{blue!55!black}{#1}}%
  \endgroup
}

\newcommand{\deploybox}[1]{%
  \begingroup
  \setlength{\fboxsep}{1.5pt}%
  \colorbox{green!12}{\textcolor{green!45!black}{#1}}%
  \endgroup
}

\newcommand{\evalbox}[1]{%
  \begingroup
  \setlength{\fboxsep}{1.5pt}%
  \colorbox{black!4}{\textcolor{black!75}{#1}}%
  \endgroup
}

\definecolor{ourpurple}{HTML}{ECE2F7}

\title{Distribution-Aware Programming: Learning Specialized Solvers from Experience}

\author{%
  Saharsh Koganti \\
  Texas A\&M University \\
  \texttt{saharshk11@tamu.edu} \\
  \And
  Priyadarsi Mishra \\
  Texas A\&M University \\
  \texttt{priyadarsimishra@tamu.edu} \\
  \AND
  Pierfrancesco Beneventano \\
  Massachusetts Institute of Technology \\
  \texttt{pierb@mit.edu} \\
  \And
  Tomer Galanti \\
  Texas A\&M University \\
  \texttt{galanti@tamu.edu} \\
}

\hypersetup{
  pdftitle={Distribution-Aware Programming: Learning Specialized Solvers from Experience},
  pdfauthor={},
  pdfsubject={ICLR 2027 manuscript},
  hidelinks
}

\begin{document}

\maketitle

\begin{abstract}
Many optimization problems are solved repeatedly on instances drawn from the same underlying distribution. In this setting, a system can begin with a general-purpose solver and use experience from previous instances to learn a cheaper way to solve future ones. We formalize this as \emph{distribution-aware programming}: samples from an unknown deployment distribution are used to produce executable solver code whose quality and runtime generalize to new instances. A simple analysis shows how sample access interpolates between distribution-oblivious and distribution-informed algorithm design, and when the offline cost of specialization is recovered through lower deployment cost. We instantiate this framework with an LLM agent that proposes structural hypotheses, analyzes training instances, and synthesizes specialized solver code; the LLM is used only before deployment. Across \(21\) structured combinatorial-optimization distributions, the synthesized solvers achieve high quality while often replacing generic search or optimization with smaller distribution-specific computations. Against released PACE competition solvers, our method remains competitive while using substantially less runtime: \(75\)--\(125\times\) less on PACE 2025 Dominating Set, \(80\)--\(96\times\) less on Hitting Set, and more than \(34{,}000\times\) less on the PACE 2024 OCM exact track. More broadly, the results suggest using AI not only to solve optimization problems, but to learn how recurring distributions should be solved.
\end{abstract}

\section{Introduction}

The fields of optimization and theoretical computer science typically
focus on how efficiently a problem can be solved, very often in the
worst case.
We tackle the common but slightly different setting in which we
repeatedly receive new problem instances sampled from a particular
distribution or drawn from a specific subfamily of problems. In this
setting, the goal is to design an algorithm that minimizes cumulative
computation across all such instances.
This perspective reveals a strong connection between optimization and
learning theory, which we explore further in
Section~\ref{sec:analysis}. With a few sample instances available
before deployment, the central observation is:
\begin{center}
    \textit{The question of average-time complexity can be tackled
    through statistical learning theory.}
\end{center}
Understanding average-time complexity then corresponds to finding,
within a hypothesis class of algorithms, one that minimizes expected
risk in terms of runtime and correctness.
This aligns philosophically with the idea that, when the learned
hypothesis is an executable procedure, runtime should matter, not just quality. If two students learn multiplication equally well, but one discovers a faster method, we would say that student learned a better procedure.
We term this paradigm \emph{distribution-aware programming}: through empirical risk minimization, find an optimally generalizing (lowest average-time complexity) solver for your distribution of problems. In this article we formalize this setting and discuss when it is tractable and useful as a point of view. After that, we ask:
\begin{center}
\textit{Can a system based on publicly available LLMs learn this at inference time? \\ If so, what is a good system design?}
\end{center}
The positive results of our experiments have three, we believe, strong implications.
First, this perspective suggests a different role for AI in
optimization. Much recent work places learned models \emph{inside}
optimization procedures, using them to predict, branch, search,
or reason on each new instance. We instead use AI as the engineer
behind the solver, not as part of the solver loop itself.

Second, this targets the tradeoff between general-purpose solvers,
which handle broad input classes, and customized solvers, which
leverage structure but require human design. We ask
whether this specialization can instead be learned from experience,
and show that \textit{LLM-based systems can do so}.

Third, the cost of solving a problem with unknown structure for the
first time is very different from the cost of solving further
instances from the same family. As experience accumulates, parts
of the generic computation can be replaced by computation tailored
to the recurring workload. In this sense, the process is
\emph{meta-reasoning}: general-purpose reasoning is used not to
solve instances, but to determine how future instances should be
solved. Once synthesis is complete, the LLM agent disappears;
future instances are handled by ordinary executable code.
Expensive general-purpose reasoning is performed offline and
compiled into cheaper specialized computation that can be reused
many times. In this work, the principle applies both to the
generic-to-customized solver pipeline and to the LLM reasoning
used to construct those solvers: we improve performance and reduce
cost by hard-coding recurring reasoning into our harness.

{\bf Contributions.\enspace} Our main contributions are:

\begin{itemize}[leftmargin=1.5em,itemsep=3pt,topsep=3pt]

\item \textbf{Distribution-aware programming.}
We formulate repeated algorithm design as a precise learning problem: samples from an unknown deployment distribution are used to synthesize executable solver code, evaluated by solution quality and deployment runtime. We formalize \textit{solver hints} as a sufficient condition for tractability of our distribution of problems.

\item \textbf{Sample-access complexity.}
We formalize the value of observing previous instances through distribution-oblivious, distribution-informed, and sample-conditioned complexity. For bounded program classes, the sample-conditioned cost approaches the distribution-informed optimum as the sample grows, while repeated deployment can amortize the offline learning cost.

\item \textbf{LLM-based solver synthesis.}
We develop an agentic procedure that proposes structural hypotheses, extracts reusable summaries from samples, and synthesizes and refines deployment solvers. The resulting solver requires no LLM inference at deployment.

\item \textbf{Empirical specialization and external validation.}
Across \(21\) structured distributions spanning seven optimization problems, the synthesized solvers often replace generic search with smaller distribution-specific computations. On released PACE 2025 Dominating Set and Hitting Set and PACE 2024 OCM benchmarks, the same procedure produces fully valid solvers with large runtime advantages over competition code at modest objective gaps.

\end{itemize}

%
%
\definecolor{ourline}{HTML}{1B6E8C}   
\definecolor{ourtext}{HTML}{114C63}   
\definecolor{ourfill}{HTML}{EDF5F8}   
\definecolor{panelbg}{HTML}{F4F5F6}   
\definecolor{linegray}{HTML}{C3C8CC}  

\begin{figure}[t]
\centering
\begin{adjustbox}{max width=\linewidth,center}
    \begin{tikzpicture}[
        font=\footnotesize, >=Latex,
        box/.style={
            draw=linegray, line width=0.5pt, rounded corners=2.5pt,
            align=center, text width=3.85cm, minimum height=0.80cm,
            inner sep=3.2pt, fill=white
        },
        ours/.style={
            box, draw=ourline!70, line width=0.7pt, fill=ourfill
        },
        title/.style={
            font=\footnotesize\bfseries, align=center, text width=4.1cm,
            text=black!78
        },
        ourstitle/.style={title, text=ourtext},
        rowlbl/.style={
            font=\scriptsize\itshape, anchor=east, text=black!50,
            align=right, text width=1.8cm
        },
        arrow/.style={
            ->, line width=0.6pt, shorten >=1.5pt, shorten <=1.5pt,
            draw=black!45
        },
        oursArrow/.style={arrow, line width=0.8pt, draw=ourline},
        arrlbl/.style={font=\scriptsize\itshape, text=ourtext, inner sep=1.5pt},
    ]

        \def\xA{0}      \def\xB{5.07}  \def\xC{10.14}
        \def\pw{2.19}                              
        \def\yTop{2.78} \def\yBot{-1.68}           
        \def\yTitle{2.40}
        \def\yRule{2.12}
        \def\yInfo{1.36}
        \def\yChoice{0.10}
        \def\yCost{-1.12}

        \begin{scope}[on background layer]
            \fill[panelbg, rounded corners=5pt]
                (\xA-\pw,\yBot) rectangle (\xA+\pw,\yTop);
            \fill[panelbg, rounded corners=5pt]
                (\xC-\pw,\yBot) rectangle (\xC+\pw,\yTop);
            \fill[ourline!10, rounded corners=5pt]
                (\xB-\pw,\yBot) rectangle (\xB+\pw,\yTop);
            \draw[ourline!45, line width=0.7pt, rounded corners=5pt]
                (\xB-\pw,\yBot) rectangle (\xB+\pw,\yTop);
        \end{scope}

        \draw[->, line width=0.6pt, draw=black!35]
            (\xA-\pw,3.22) -- (\xC+\pw,3.22);
        \node[font=\scriptsize\itshape, text=black!55, fill=white,
              inner sep=2.5pt] at (\xB,3.22)
            {increasing information about \(D\)};

        \node[rowlbl] at (\xA-\pw-0.35,\yInfo)   {information};
        \node[rowlbl] at (\xA-\pw-0.35,\yChoice) {solver choice};
        \node[rowlbl] at (\xA-\pw-0.35,\yCost)   {complexity};

        \node[title]     at (\xA,\yTitle) {Distribution-oblivious};
        \node[ourstitle] at (\xB,\yTitle) {Sample access};
        \node[title]     at (\xC,\yTitle) {Distribution known};

        \draw[linegray] (\xA-1.90,\yRule) -- (\xA+1.90,\yRule);
        \draw[ourline!55, line width=0.7pt]
                        (\xB-1.90,\yRule) -- (\xB+1.90,\yRule);
        \draw[linegray] (\xC-1.90,\yRule) -- (\xC+1.90,\yRule);

        \node[font=\tiny\bfseries, text=white, fill=ourline,
              rounded corners=2pt, inner xsep=4pt, inner ysep=1.6pt]
            at (\xB,\yTop) {THIS WORK};

        \node[box] (A1) at (\xA,\yInfo)   {no access to\\\(D\in\mathcal D\)};
        \node[box] (A2) at (\xA,\yChoice) {commit to one \(A\in\mathcal A\)\\
                                           before deployment};
        \node[box] (A3) at (\xA,\yCost)   {\scriptsize\(\displaystyle C_0=
                                           \inf_{A\in\mathcal A}\,
                                           \sup_{D\in\mathcal D}\,R_D(A)\)};
        \draw[arrow] (A1) -- (A2);
        \draw[arrow] (A2) -- (A3);

        \node[ours] (B1) at (\xB,\yInfo)   {sample \(S\sim D^{m}\)
                                            from unknown \(D\)};
        \node[ours] (B2) at (\xB,\yChoice) {learner selects\\
                                            \(L(S)\in\mathcal A\)};
        \node[ours] (B3) at (\xB,\yCost)   {\scriptsize\(\displaystyle
                                            \textcolor{ourtext}{C_m}=\inf_{L}\,
                                            \sup_{D\in\mathcal D}\,
                                            \mathbb E_S[R_D(L(S))]\)};
        \draw[oursArrow] (B1) -- (B2) node[arrlbl, midway, right=2pt] {learn};
        \draw[oursArrow] (B2) -- (B3);

        \node[box] (C1) at (\xC,\yInfo)   {\(D\in\mathcal D\)\\known exactly};
        \node[box] (C2) at (\xC,\yChoice) {pick the best\\
                                           \(A\in\mathcal A\) for \(D\)};
        \node[box] (C3) at (\xC,\yCost)   {\scriptsize\(\displaystyle C_{\mathrm{known}}=
                                           \sup_{D\in\mathcal D}\,
                                           \inf_{A\in\mathcal A}\,R_D(A)\)};
        \draw[arrow] (C1) -- (C2);
        \draw[arrow] (C2) -- (C3);

        \node[font=\small, text=ourtext] at ($(\xA,\yCost)!0.5!(\xB,\yCost)$)
            {\(\geq\)};
        \node[font=\small, text=ourtext] at ($(\xB,\yCost)!0.5!(\xC,\yCost)$)
            {\(\geq\)};

        \draw[->, ourline, line width=0.7pt, dashed, dash pattern=on 2pt off 1.6pt,
              rounded corners=3pt]
            ($(B3.south)+(0,-0.05)$) -- ($(B3.south)+(0,-0.42)$)
            -- ($(C3.south)+(0,-0.42)$) -- ($(C3.south)+(0,-0.05)$);
        \node[font=\scriptsize\itshape, text=ourtext, fill=white, inner sep=2pt]
            at ($(B3.south)!0.5!(C3.south)+(0,-0.42)$)
            {\(m\to\infty\)};

    \end{tikzpicture}
\end{adjustbox}

\caption{%
\textbf{Three levels of access to the deployment distribution.}
With no information about \(D\), a single solver must be fixed before
deployment, giving \(C_0\). If \(D\) were known exactly, the solver could be
specialised to it, giving \(C_{\mathrm{known}}\). We study the intermediate
regime (centre): \(D\) is unknown, but the learner observes a sample
\(S\sim D^m\) and selects \(L(S)\in\mathcal A\). Proposition~%
\ref{prop:sample_conditioned_complexity} shows
\(C_{\mathrm{known}}\le C_m\le C_0\), and that \(C_m\to C_{\mathrm{known}}\)
for a fixed bounded-description program class.%
}
\label{fig:teaser_concept}
\end{figure}


\section{Related Work}

\vspace{-0.5cm}

{\bf Distribution-dependent algorithm design.\enspace}
Average-case complexity, smoothed analysis, and beyond-worst-case
analysis explain why hard problems may admit efficient algorithms
under particular input distributions or structural assumptions
\citep{levin1986average,impagliazzo1995personal,
spielman2004smoothed,roughgarden2019beyond}.
Parameterized and backdoor complexity likewise identify structure
that can make otherwise hard problems tractable
\citep{downey2013fundamentals,williams2003backdoors}.
Speedup learning studies how experience improves the computational
efficiency of problem solving \citep{tadepalli1996speedup}.
Most closely related are self-improving algorithms, which learn
from samples of an initially unknown distribution to accelerate
future computation
\citep{ailon2011self,clarkson2014selfimproving,
cheng2020generalization}.
Our setting follows the same amortization principle but does not
assume a hand-designed representation of useful structure or a
problem-specific self-improving algorithm: the hint, analysis
procedure, and specialized solver may themselves be synthesized.

{\bf Average-case optimization and practical specialization.\enspace}
Classical probabilistic analyses obtain polynomial expected bounds
for simplex, complementary pivoting, and local-improvement methods
under specified input models
\citep{borgwardt1982average,smale1983average,todd1986pivoting,
tovey1986local}.
Modern average-case optimization relates convergence and halting
times to spectral structure and derives distribution-adapted methods
\citep{pedregosa2020spectral,scieur2020polyak,paquette2023halting,
paquette2021sgd,cunha2022tails,domingoenrich2021bilinear}.
In practice, automated tuning targets runtime, solution quality,
or anytime performance on representative workloads
\citep{adensodiaz2006calibra,hutter2010mip,
lopezibanez2014anytime,lopezibanez2016irace};
commercial solvers such as Gurobi support average-performance
tuning across multiple models \citep{gurobi2026}.
Specialization can change the computational procedure itself:
generated heuristics, explicit control laws, specialized solver
code, learned solution strategies, and warm starts offer different
routes to reducing repeated-solve costs
\citep{mascia2014grammar,bemporad2002explicit,
mattingley2012cvxgen,bertsimas2022milliseconds,
sambharya2024warmstart}.

{\bf Algorithm selection and learning-augmented algorithms.\enspace}
Algorithm selection and configuration exploit the fact that
different solvers perform best on different instances or distributions
\citep{rice1976algorithm,xu2008satzilla,xu2010hydra,
lindauer2015autofolio}.
PAC analyses characterize the sample complexity of selecting
high-performing algorithms for an unknown instance distribution
\citep{gupta2017pac,balcan2017configuration,balcan2021data}.
Complementary work bounds the computational cost of configuration
through adaptive evaluation and capped runs, with connections
to bandit allocation
\citep{kleinberg2017procrastination,weisz2018leapsandbounds,
weisz2019capsandruns,kleinberg2019procrastinating}.
Learning-augmented algorithms incorporate learned predictions
into classical algorithms while seeking robustness to prediction
error
\citep{lykouris2018competitive,mitzenmacher2022algorithms}.
Our synthesis procedure uses samples before deployment to
construct the specialization itself: the structural hypothesis,
the analysis program that measures its supporting evidence, and
the solver code conditioned on the resulting summary.

{\bf Learned computation for optimization.\enspace}
Neural combinatorial optimization learns policies, heuristics,
branching rules, or representations for hard optimization problems
\citep{vinyals2015pointer,bello2017neural,khalil2017learning,
gasse2019exact,bengio2021machine,cappart2023combinatorial}.
Related work on adaptive and test-time computation varies the
amount of inference, search, or reasoning performed on each input
\citep{teerapittayanon2016branchynet,wang2023selfconsistency,
yao2023tree,snell2025scaling}.
Our emphasis is different: expensive learned reasoning occurs
before deployment and changes the solver itself. The resulting
program may be symbolic or combinatorial and requires no model
inference on new instances.

{\bf Program synthesis and code generation.\enspace}
Program synthesis from examples has long studied how observations
or specifications can be turned into executable programs
\citep{gulwani2011flashfill,balog2017deepcoder,
devlin2017robustfill,nye2019sketches}, while large language models
have greatly expanded the scope of code generation and autonomous
software agents
\citep{hendrycks2021apps,li2022alphacode,jimenez2024swebench,
zhuo2025bigcodebench,huang2024mlagentbench,chan2025mlebench}.
LLM-guided algorithm discovery also uses execution-based feedback
to generate and improve reusable heuristics and algorithms
\citep{romeraparedes2024funsearch,liu2024eoh,ye2024reevo,
novikov2025alphaevolve}.
Our emphasis is on using samples from a repeated computational
workload to synthesize an algorithm whose \emph{computation}
is specialized to that workload.

{\bf Learning executable hypotheses.\enspace}
Classical learning theory provides the general framework for
learning hypotheses that generalize from samples
\citep{valiant1984theory,vapnik1971uniform,vapnik1998statistical,
shaibook}, and recent work has considered empirical risk
minimization over programs under LLM-induced priors
\citep{singhal2025llmpriors}.
We evaluate executable hypotheses by both solution quality and
runtime. Our analysis compares distribution-oblivious,
sample-conditioned, and distribution-informed computational costs,
and accounts for the amortization of offline learning across
repeated solves.

\section{Learning and Computational Cost}
\label{sec:analysis}

We begin with a very simple comparison of three settings: choosing a solver without knowing the deployment distribution, choosing one when the distribution is known, and choosing one after observing samples. This also shows when computation spent before deployment can be recovered through lower runtime on future instances.

{\bf Defining Distribution-Aware Programming.\enspace}
Fix an input size \(n\), let \(\mathcal D_n\) be a family of distributions supported on instances of size \(n\), and let \(\mathcal A\) be a fixed class of correct solver programs. For \(A\in\mathcal A\), write \(R_D(A)=\mathbb E_{X\sim D}[T(A,X)]\) for its expected runtime under \(D\). Going forward, we suppress the dependence on \(n\).
If the deployment distribution is unknown, the solver must perform well uniformly over \(\mathcal D_n\). If the distribution were known, one could instead choose a solver specialized to it. These two settings correspond to
\[
    C_0=\inf_{A\in\mathcal A}\sup_{D\in\mathcal D_n}R_D(A),
\qquad
C_{\mathrm{known}}
=\sup_{D\in\mathcal D_n}\inf_{A\in\mathcal A}R_D(A).
\]
The difference is the order in which the algorithm and distribution are chosen. By the minimax inequality, \(C_{\mathrm{known}}\le C_0\), so \(C_0-C_{\mathrm{known}}\) measures the potential computational benefit of adapting the solver to the deployment distribution.

In practice, the distribution is not given explicitly. Instead, we may observe previous instances from the same source. Let \(S=(X_1,\ldots,X_m)\sim D^m\), and let a learning procedure use \(S\) to choose a solver for future independent draws from \(D\). Define \(C_m=\inf_L\sup_{D\in\mathcal D_n}\mathbb E_S[R_D(L(S))]\), where the infimum is over learning procedures and the expectation includes any randomness in the learner. The question is whether samples can move the computational cost from the distribution-oblivious quantity \(C_0\) toward the distribution-informed quantity \(C_{\mathrm{known}}\).

{\bf Samples help.\enspace}
Suppose the possible deployed programs have binary descriptions of length at most \(\ell\) in a fixed programming language. The class is fixed before observing the sample, although the selected program may depend on it. There are at most \(N\le 2^{\ell+1}-1\) such programs. We also allow randomized choices among them; for any fixed \(D\), randomization cannot improve expected runtime beyond that of the best base program. Using standard statistical learning theory tools~\citep{shaibook}:

\begin{restatable}[Sample-conditioned computational complexity]
{proposition}{propsampleconditionedcomplexity}
\label{prop:sample_conditioned_complexity}
Suppose \(0\le T(A,X)\le B\) for every base program \(A\), every
\(D\in\mathcal D_n\), and every \(X\) in the support of \(D\). For \(m\ge 1\), let \(\varepsilon_m = B\sqrt{\frac{(\ell+1)\log 2}{2m}}\). Then \(C_{\mathrm{known}} \le C_m \le \min\!\left\{C_0,\, C_{\mathrm{known}}+\varepsilon_m \right\}\). In particular, for a fixed program class, \(C_m\to C_{\mathrm{known}}\) as \(m\to\infty\).
\end{restatable}

The proposition shows that samples can approach the computational benefit of knowing the distribution without recovering \(D\) itself: it suffices to identify a program with low expected runtime. If \(F(m)\) bounds the offline selection cost, reuse over \(Q\) instances gives amortized cost at most
\(F(m)/Q+C_{\mathrm{known}}+\varepsilon_m\), improving on \(C_0\) whenever
\(F(m)/Q+\varepsilon_m<C_0-C_{\mathrm{known}}\).

The remaining challenge is search: bounded description length controls
generalization but does not make the program class tractable to explore.
We therefore use LLM-guided search over structural hypotheses and solver
programs. Given the task specification and public samples, the agent proposes
recurring structure, estimates it with an analysis program, and generates a
solver conditioned on the resulting summary. Validation quality and runtime
guide refinement. The analysis above idealizes selection from a fixed class of
correct programs; our experiments test whether this guided search can find
fast, high-quality specializations in practice.

\section{Method}
\label{sec:method}

Our goal is to infer reusable structure from sample instances and compile it into
a solver for the same deployment distribution. Because neither the useful
structure nor the program that exploits it is known in advance, we use an LLM
agent to search jointly over structural hypotheses, analysis programs, and
deployment solvers. Figure~\ref{fig:method_pipeline} summarizes the procedure;
full prompts, implementation details, validation checks, and failure handling
appear in Appendix~\ref{app:method_details}.

\begin{figure}[t]
\centering
\begin{adjustbox}{max width=\linewidth,center}
\begin{tikzpicture}[
    font=\scriptsize, >=Stealth,
    box/.style={
        draw=black!45, line width=0.5pt, rounded corners=2pt,
        align=center, inner sep=3pt, minimum height=0.66cm,
        text width=1.55cm, fill=white, font=\scriptsize
    },
    data/.style={box, fill=black!5, draw=black!40},
    hyp/.style={box, fill=hypbg, draw=hypfg!65, text=hypfg},
    ana/.style={box, fill=anabg, draw=anafg!65, text=anafg},
    sol/.style={box, fill=solbg, draw=solfg!65, text=solfg},
    hint/.style={
        box, fill=ourhint, draw=ourline, line width=0.7pt,
        text=ourtext, text width=1.5cm
    },
    depl/.style={
        box, fill=ourfill, draw=ourline, line width=0.7pt,
        text=ourtext, text width=1.7cm
    },
    arr/.style={
        -{Stealth[length=3.6pt,width=3.6pt]}, line width=0.55pt,
        draw=black!65, shorten >=1pt, shorten <=1pt
    },
    aarr/.style={arr, line width=0.8pt, draw=ourline},
    darr/.style={
        arr, line width=0.7pt, draw=black!50,
        dash pattern=on 2.6pt off 1.8pt
    },
    stg/.style={font=\tiny\bfseries},
    lbl/.style={font=\tiny\itshape, text=black!55, inner sep=1.5pt},
    olbl/.style={font=\tiny\itshape, text=ourtext, inner sep=1.5pt},
    ptitle/.style={font=\scriptsize\bfseries, align=left, anchor=west},
    psub/.style={font=\tiny\itshape, text=black!55, align=left, anchor=west},
]

    \def\xS{0.70}  \def\xH{2.72}  \def\xA{4.74}
    \def\xC{6.76}  \def\xE{8.78}  \def\xR{10.80}
    \def\xD{14.75}
    \def\yA{1.05}
    \def\yH{-0.35}
    \def\yF{-1.45}
    \def\pL{-0.25}  \def\pR{11.75}
    \def\qL{13.05}  \def\qR{16.45}
    \def\pB{-1.95}  \def\pT{2.72}

    \begin{scope}[on background layer]
        \fill[black!3, rounded corners=5pt]
            (\pL,\pB) rectangle (\pR,\pT);
        \fill[ourline!8, rounded corners=5pt]
            (\qL,\pB) rectangle (\qR,\pT);
        \draw[ourline!40, line width=0.6pt, rounded corners=5pt]
            (\qL,\pB) rectangle (\qR,\pT);

        \filldraw[
            fill=white, draw=black!22, line width=0.4pt,
            rounded corners=3pt
        ]
            ($(\xH-0.95,\yA-0.45)+(0.20,0.20)$)
            rectangle
            ($(\xC+0.95,\yA+0.72)+(0.20,0.20)$);

        \filldraw[
            fill=white, draw=black!22, line width=0.4pt,
            rounded corners=3pt
        ]
            ($(\xH-0.95,\yA-0.45)+(0.10,0.10)$)
            rectangle
            ($(\xC+0.95,\yA+0.72)+(0.10,0.10)$);

        \filldraw[
            fill=white, draw=black!30, line width=0.4pt,
            rounded corners=3pt
        ]
            (\xH-0.95,\yA-0.45)
            rectangle
            (\xC+0.95,\yA+0.72);
    \end{scope}

    \node[ptitle] at (\pL+0.2,\pT-0.24)
        {Synthesis \textnormal{\textit{(offline)}}};
    \node[psub] at (\pL+0.2,\pT-0.52)
        {LLM proposes and writes; candidates are executed and evaluated};

    \node[ptitle, text=ourtext] at (\qL+0.2,\pT-0.24)
        {Deployment};
    \node[psub, text=ourtext!85] at (\qL+0.2,\pT-0.52)
        {new instances, no LLM};

    \node[lbl, anchor=east, text=black!45]
        at (\xC+1.15,\yA+1.10)
        {\(K\) candidates per round};

    \node[stg, text=hypfg] at (\xH,\yA+0.52)
        {1. hypothesis};
    \node[stg, text=anafg] at (\xA,\yA+0.52)
        {2. analysis};
    \node[stg, text=solfg] at (\xC,\yA+0.52)
        {3. solver};

    \node[data] (S) at (\xS,\yA)
        {public samples\\\(S_{\mathrm{tr}},S_{\mathrm{val}}\)};
    \node[hyp] (H) at (\xH,\yA)
        {hypothesis\\\(H_c\)};
    \node[ana] (A) at (\xA,\yA)
        {analysis\\\(A_c\)};
    \node[sol] (C) at (\xC,\yA)
        {solver\\\(s_c\)};
    \node[data] (E) at (\xE,\yA)
        {evaluate\\on \(S_{\mathrm{val}}\)};
    \node[data] (R) at (\xR,\yA)
        {rank +\\update beam};

    \node[hint] (hh) at (\xA,\yH)
        {hint \(a_c\)};

    \draw[arr] (S) -- (H);
    \draw[arr] (H) -- (A);
    \draw[arr] (A) -- (C);
    \draw[arr] (C) -- (E);
    \draw[arr] (E) -- (R);

    \draw[aarr] (A) -- (hh)
        node[olbl, midway, right=1pt] {run on \(S_{\mathrm{tr}}\)};

    \draw[aarr] (hh.east) -- (\xC,\yH) -- (C.south)
        node[olbl, pos=0.55, right=1pt] {condition};

    \node[lbl, text=black!50] at (\xE,\yA-0.62)
        {quality, optimality, runtime};

    \node[olbl, text=ourtext!85] at (\xA,\yH-0.52)
        {reusable distribution summary};

    \draw[darr]
        (R.south) |- (\xH+2.6,\yF) -- (\xH,\yF) -- (H.south);

    \node[lbl, fill=black!3, inner sep=2pt]
        at (\xH+2.85,\yF)
        {refine / fork / replace \(\;\cdot\;\)
         \(R\) rounds, beam \(B\), \(K\) candidates};

    \draw[aarr, rounded corners=3pt]
        (R.east) -- (\pR+0.45,\yA)
        -- (\pR+0.45,\yH)
        -- (\xD-0.90,\yH);

    \node[olbl, align=center]
        at (\pR+1.20,\yH+0.34)
        {deploy \(\widehat s,\widehat a\)};

    \node[olbl, align=center, text=black!55]
        at (\pR+0.45,\yA+0.42)
        {best candidate};

    \node[data] (x) at (\xD,\yA)
        {new instance\\\(x\sim D\)};
    \node[depl] (dep) at (\xD,\yH)
        {\(\widehat s(x,\widehat a)\)};
    \node[data] (z) at (\xD,\yF+0.12)
        {solution};

    \draw[arr] (x) -- (dep);
    \draw[aarr] (dep) -- (z);

\end{tikzpicture}
\end{adjustbox}

\caption{
\textbf{Offline synthesis and deployment.}
For each candidate, the LLM proposes a hypothesis \(H_c\), writes an analysis
program \(A_c\), and uses its output \(a_c\) to generate a solver \(s_c\).
Candidates are evaluated and refined over several rounds. The selected solver
\(\widehat s\) and hint \(\widehat a\) are then reused on new instances without
LLM inference.
}
\label{fig:method_pipeline}
\end{figure}

{\bf Candidate construction and search.\enspace}
Each candidate consists of a structural hypothesis \(H_c\), an analysis program
\(A_c\), and a deployment solver \(s_c\). The analysis program maps the training
sample to a reusable hint \(a_c=A_c(S_{\rm tr})\), and the solver uses
\((x,a_c)\) on new instances. The hint representation is task-dependent and may
encode bottlenecks, vertex roles, geometric structure, recurring Boolean
patterns, or other reusable structure.

We maintain a beam of candidates over several rounds. Candidates are ranked
lexicographically by validation quality, optimality rate, and runtime,
\((Q_{\rm val},O_{\rm val},-T_{\rm val})\). Initial candidates explore diverse
structural explanations; later rounds may refine a promising candidate, branch
to another explanation, replace a weak hypothesis, or improve quality or
runtime using validation feedback and representative failures. To preserve
diversity, we retain the best candidate from each hypothesis class before
filling the remaining beam slots by validation score. After the final round, we
select the best candidate seen during the search, rerun its analysis on the
training sample, and deploy the resulting solver. Invalid or failed candidates
receive zero quality. Full prompts, scoring, diversity rules, and failure
handling appear in Appendix~\ref{app:method_details}.

\section{Experiments}
\label{sec:experiments}

We evaluate whether the synthesis procedure of Section~\ref{sec:method} can identify reusable structure from samples and turn it into solvers with better quality--runtime tradeoffs. We compare against classical heuristics and optimization solvers, learned neural baselines, and alternative LLM-based synthesis methods. Full experimental details, benchmark definitions, per-target results, and additional ablations are provided in the appendix.

\subsection{Experimental setup}

{\bf Benchmarks.\enspace}
We evaluate on both controlled and external benchmarks. Our controlled suite contains \(21\) structured target distributions spanning seven combinatorial optimization problems, with three hidden distribution families per problem. Each target contains \(64\) training, \(32\) validation, and \(500\) held-out test instances. The learner sees only the public problem instances and specification, not the hidden family identity, generative rule, optimum solutions, or optimum objective values.

We also evaluate on released PACE competition benchmarks: PACE 2025 Dominating Set and Hitting Set~\citep{grobler_et_al:LIPIcs.IPEC.2025.32} and PACE 2024 One-Sided Crossing Minimization (OCM)~\citep{kindermann}. Released public instances are used for synthesis and model selection, while released private instances are reserved for evaluation. These substantially larger, independently constructed datasets provide an external test against highly engineered competition solvers. Detailed Dominating Set statistics, protocol information, and matched-budget results appear in Appendix~\ref{app:pace_ds}.

\subsubsection{Baselines}

We compare our synthesis agent (Section~\ref{sec:method}) against classical heuristics and optimization solvers, learned neural methods, and alternative LLM-based synthesis procedures. All methods are evaluated on the same public instances and receive no information about the hidden distribution family or its generative rule.

{\bf Classical baselines.\enspace}
For each problem class, we compare against standard problem-specific heuristics and general-purpose or exact optimization systems (e.g.,~\citealp{brelaz1979new,chvatal1979greedy,dantzig1957discrete,helsgaun2000effective,gurobi2026,perron2025ortools,bestuzheva2023scip,highs2024}). The complete catalog appears in Appendix~\ref{app:baseline_catalog_protocol}.

{\bf LLM baselines.\enspace}
We compare against one-shot solver synthesis with OpenAI Codex~\citep{openai2025codex} and Anthropic Claude Code~\citep{anthropic2025claudecode}, using the same problem specification and tool access but without iterative refinement. We also use a Gemma~4 best-of-\(5\) baseline (\texttt{gemma-4-31b-it}): five independent solvers are generated from the same specification and the best is selected on validation, matching the candidate budget without the refinement loop.

{\bf Learned ML baselines.\enspace}
We also compare against one neural baseline for each problem class: physics-inspired GNNs for coloring and independent set~\citep{schuetz2022graphcoloring,schuetz2022combinatorial}, a RUN-CSP-style model for MaxSAT~\citep{toenshoff2021runcsp}, a GNN with reinforcement learning for dominating set~\citep{chen2024dominatingset}, an attention model for TSP~\citep{kool2019attention}, a reinforcement-learning baseline for multidimensional knapsack~\citep{bushaj2024kmeans}, and a primal--dual model for packing LP~\citep{park2023selfsupervised}. These implementations follow the respective model families rather than reproducing the cited systems exactly. They were trained and evaluated on reduced-size per-target datasets drawn from the same target distributions, rather than the full 500-instance test sets used for the main classical comparisons; we therefore treat them as auxiliary model-family baselines. We report inference time on both CPU (AMD EPYC 9554) and GPU (NVIDIA H100 NVL); training time is excluded from deployment runtime, just as synthesis time is excluded for our method.

\subsubsection{Evaluation}

The main metric is normalized solution quality, where \(1\) is optimal and larger is better; invalid or infeasible outputs receive quality \(0\). We also report optimality rate, feasibility rate, and per-instance wall-clock runtime. The problem-specific quality definitions are given in Table~\ref{tab:quality_metrics} of Appendix~\ref{app:additional_details}. Results are first averaged over held-out test instances and repeated runs within each target, and then averaged across the \(21\) targets.

Let \(Q_{\rm ours}\) denote the quality of our synthesized solver and \(Q_b\) the quality of baseline \(b\). We report the quality difference \(\Delta Q_b = Q_{\rm ours}-Q_b\), so positive values mean our solver has higher quality. Comparisons include Gemma best-of-\(5\), one-shot Codex and Claude Code, the average heuristic, the learned ML baseline, and a \emph{fast high-quality heuristic}. The latter is the fastest heuristic whose average quality is at least \(95\%\) of the best heuristic for that problem class. Timeouts receive quality \(0\) and the \(360\)s runtime. Invalid or failed Gemma best-of-\(5\) outputs are likewise scored as zero quality.

For runtime, we report the speedup \(S_b = T_b/T_{\rm ours}\), where \(T_b\) and \(T_{\rm ours}\) are the baseline and synthesized-solver runtimes. Thus \(S_b>1\) means our solver is faster. Speedups are computed separately for each target and combined using the geometric mean. We compare against the LLM baselines, CPU and GPU inference for the learned ML models, the fast high-quality heuristic, Gurobi, and the time-limited exact baseline. Failed Gemma best-of-\(5\) runs are excluded from runtime aggregation because their measured time reflects immediate failure rather than solver execution.


\subsection{Results}
\label{sec:main_results}

\begin{table}[t]
\centering
\caption{
\textbf{Quality and runtime across the \(21\) target distributions.}
\(Q_{\rm ours}\) and \(T_{\rm ours}\) are the quality and runtime of our
synthesized solver. For each baseline \(b\),
\(\Delta Q_b = Q_{\rm ours}-Q_b\), so positive values favor our solver, and
\(S_b=T_b/T_{\rm ours}\), so \(S_b>1\) means our solver is faster.
Quality is normalized so that \(1\) is optimal.
}
\label{tab:headline_results}

\scriptsize
\setlength{\tabcolsep}{1.7pt}
\renewcommand{\arraystretch}{1.08}

\begin{adjustbox}{max width=\linewidth,center}
\begin{tabular}{
@{}l
>{\columncolor{ourpurple}}c
>{\columncolor{ourpurple}}c
*{6}{c}
*{8}{c}
@{}}
\toprule

& \multicolumn{2}{c}{\textbf{Ours}}
& \multicolumn{6}{c}{\textbf{Quality difference} \(\Delta Q_b\)}
& \multicolumn{8}{c}{\textbf{Speedup} \(S_b=T_b/T_{\rm ours}\)}
\\

\cmidrule(lr){2-3}
\cmidrule(lr){4-9}
\cmidrule(lr){10-17}

\textbf{Family}
& \(T_{\rm ours}\)
& \(Q_{\rm ours}\)
& Gem@5
& Cod
& Cld
& Avg
& Heur
& ML
& Gem@5
& Cod
& Cld
& ML-C
& ML-G
& Heur
& Gur
& Exact
\\

\midrule

Coloring
& \(2.7\)
& \(0.868\)
& \(+0.085\)
& \(-0.086\)
& \(-0.065\)
& \(+0.217\)
& \(+0.121\)
& \(-0.127\)
& \(0.62\times\)
& \(0.97\times\)
& \(59.0\times\)
& \(0.49\times\)
& \(0.76\times\)
& \(1326.3\times\)
& \(2285.6\times\)
& \(23.1\times\)
\\

MAXSAT
& \(17.1\)
& \(1.000\)
& \(+0.004\)
& \(+0.006\)
& \(+0.000\)
& \(+0.122\)
& \(+0.083\)
& \(0.000\)
& \(36.0\times\)
& \(1.52\times\)
& \(1.00\times\)
& \(1.0\times\)
& \(1.0\times\)
& \(212.2\times\)
& \(328.8\times\)
& \(1.2\times\)
\\

MIS
& \(18.8\)
& \(0.992\)
& \(+0.009\)
& \(+0.009\)
& \(+0.326\)
& \(+0.218\)
& \(+0.113\)
& \(-0.002\)
& \(1.43\times\)
& \(2.19\times\)
& \(165.8\times\)
& \(0.047\times\)
& \(0.085\times\)
& \(1904.2\times\)
& \(155.8\times\)
& \(41.2\times\)
\\

MDS
& \(13.3\)
& \(0.973\)
& \(+0.111\)
& \(-0.027\)
& \(+0.022\)
& \(+0.148\)
& \(+0.124\)
& \(+0.021\)
& \(17.1\times\)
& \(22.1\times\)
& \(4.20\times\)
& \(0.16\times\)
& \(0.42\times\)
& \(1626.8\times\)
& \(443.0\times\)
& \(21.6\times\)
\\

Packing LP
& \(3.3\)
& \(0.994\)
& \(+0.215\)
& \(-0.003\)
& \(-0.006\)
& \(+0.301\)
& \(+0.317\)
& \(+0.186\)
& \(1.75\times\)
& \(34.8\times\)
& \(30.2\times\)
& \(0.19\times\)
& \(0.30\times\)
& \(13483.7\times\)
& \(2829.1\times\)
& \(44.1\times\)
\\

MDKP
& \(94.0\)
& \(0.973\)
& \(+0.579\)
& \(-0.010\)
& \(-0.011\)
& \(+0.215\)
& \(+0.009\)
& \(0.000\)
& \(0.032\times\)
& \(4.60\times\)
& \(3.69\times\)
& \(0.037\times\)
& \(0.060\times\)
& \(46.4\times\)
& \(33.8\times\)
& \(9.6\times\)
\\

TSP
& \(14.6\)
& \(0.993\)
& \(+0.008\)
& \(-0.002\)
& \(+0.327\)
& \(+0.348\)
& \(-0.007\)
& \(+0.084\)
& \(3.34\times\)
& \(3.48\times\)
& \(104.4\times\)
& \(30.0\times\)
& \(61.0\times\)
& \(33.7\times\)
& \(117.5\times\)
& \(37.2\times\)
\\

\midrule

\textbf{All (21)}
& \(\mathbf{12.7}\)
& \(\mathbf{0.971}\)
& \(\mathbf{+0.145}\)
& \(\mathbf{-0.016}\)
& \(\mathbf{+0.085}\)
& \(\mathbf{+0.224}\)
& \(\mathbf{+0.109}\)
& \(\mathbf{+0.023}\)
& \(\mathbf{2.37\times}\)
& \(\mathbf{4.53\times}\)
& \(\mathbf{17.4\times}\)
& \(\mathbf{0.36\times}\)
& \(\mathbf{0.61\times}\)
& \(\mathbf{528.6\times}\)
& \(\mathbf{345.1\times}\)
& \(\mathbf{16.9\times}\)
\\

\bottomrule
\end{tabular}
\end{adjustbox}

\vspace{2pt}
{\scriptsize
\(T_{\rm ours}\) is in milliseconds.
Avg = average heuristic; Heur = fast high-quality heuristic;
ML-C/ML-G = ML inference on CPU/GPU; Gur = Gurobi.
}

\end{table}

Running the same synthesis pipeline with Gemma~4 gives mean quality \(0.883\) versus \(0.971\) with GPT-5.2, with deployed solvers \(1.38\times\) slower on average; full per-family results appear in Appendix~\ref{app:generator_ablation}.

{\bf Main quality--runtime results.\enspace}
Table~\ref{tab:headline_results} summarizes the quality--runtime tradeoff across the \(21\) target distributions; full per-target results are in Appendix~\ref{app:full_results}. Our synthesized solvers achieve mean normalized quality \(0.971\), improving by \(0.224\) over the average heuristic and \(0.109\) over the fast high-quality heuristic. The reduced-scale learned baselines are reported as an auxiliary comparison; on those runs, our mean quality is \(0.023\) higher. At the same time, the synthesized solvers are \(528.6\times\) faster than the fast high-quality heuristic, \(345.1\times\) faster than Gurobi, and \(16.9\times\) faster than the time-limited exact baseline, using geometric means of per-target speedups. The heuristic speedup remains \(527.2\times\) after excluding methods that hit the time limit on at least half of their instances, and \(526.7\times\) after excluding anytime and local-search methods, so the aggregate gain is not driven by budget-consuming heuristics.

Compared with other LLM-generated solvers, one-shot Codex comes closest in quality, exceeding ours by \(0.016\), but is \(4.5\times\) slower. Our solver exceeds one-shot Claude Code by \(0.085\) in quality and is \(17.4\times\) faster; Claude also times out on one MIS and one TSP target. Gemma best-of-\(5\) is \(2.4\times\) slower and trails by \(0.145\) in quality, suggesting that repeated independent sampling alone does not account for the gains from iterative synthesis.

To test whether samples mainly help select among existing algorithms, we use validation performance to choose among the baselines for each problem class. The highest-quality selector reaches \(0.9805\) mean test quality versus \(0.9705\) for ours, but is \(42.6\times\) slower. The fastest selector matching our validation quality is \(27.9\times\) slower and succeeds on only \(19/21\) targets. Full results are in Appendix~\ref{app:catalog_selection}.

The advantage is not uniform across all problems. The fast heuristic achieves slightly higher quality on TSP, and the learned ML model performs better on Coloring and is faster on several families, particularly with GPU inference. Overall, however, no baseline is both faster and higher-quality across the full suite. On MIS, a matched-budget comparison with the OnlineMIS and ReduMIS modes of KaMIS is mixed: our solver is strongest on Core-fringe, while the KaMIS heuristics are stronger on Clique-path and Motif-bridge (Appendix~\ref{app:kamis_anytime}, Table~\ref{tab:kamis_anytime}).

Against stronger learned baselines on the classes where they apply, our solvers match or exceed compact reproductions of ANYCSP and GFlowNet on \(11/12\) targets and are faster on all \(12\), by \(2.5\times\) to over \(2{,}800\times\). These matched evaluations use the same \(500\) held-out instances as our solvers; full results and implementation scope appear in Appendix~\ref{app:anycsp_gflownet}.

{\bf Iteration ablation.\enspace}
We next ask whether iterative refinement improves the deployed solver beyond the initial generated candidate. After each synthesis round, we record the test-time speedup of the best candidate found so far relative to the zero-shot solver.

\begin{figure}[t]
\centering
\begin{minipage}[c]{0.48\linewidth}
    \centering
    \includegraphics[width=\linewidth]{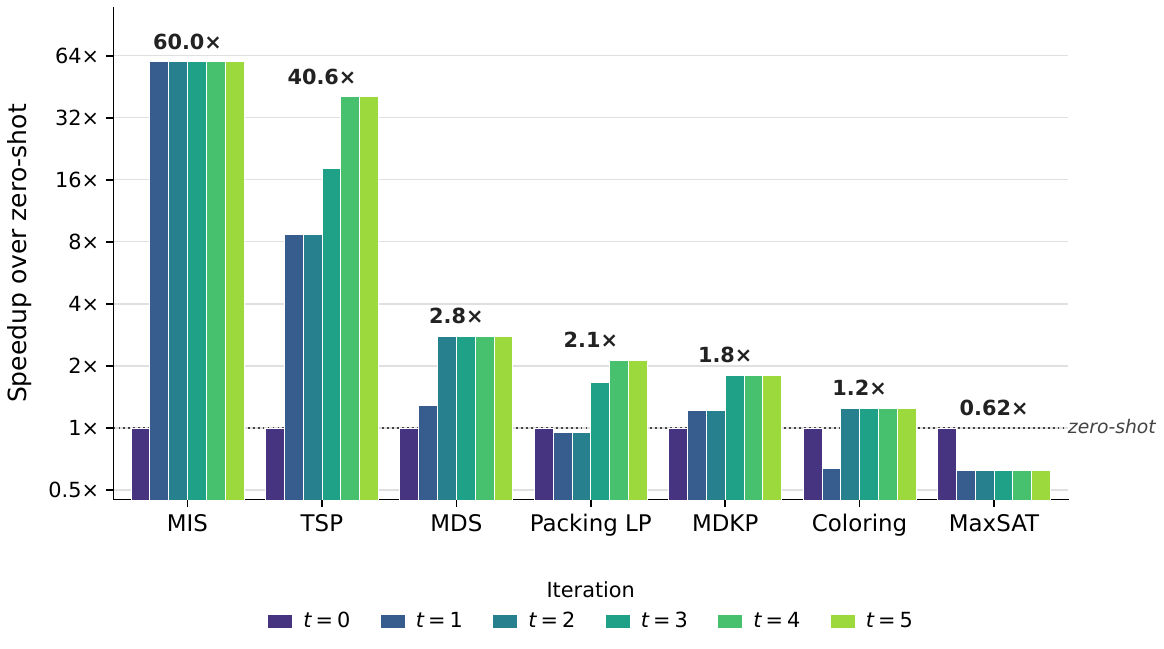}
\end{minipage}
\hfill
\begin{minipage}[c]{0.48\linewidth}
    \caption{
    \textbf{Runtime improvement across synthesis iterations.}
    Speedup is measured relative to the zero-shot generated solver; values above
    \(1\times\) indicate faster deployment. Each bar reports the best candidate
    found up to that iteration.
    }
    \label{fig:runtime-per-iteration}
\end{minipage}
\end{figure}

Figure~\ref{fig:runtime-per-iteration} shows that refinement can improve runtime. TSP, for example, increases from about \(8.7\times\) speedup after the first iteration to \(40.6\times\) by iteration \(4\), while Packing LP moves from initially slower candidates to about \(2.1\times\) speedup. MDS and MDKP also improve after the first proposal. MaxSAT is the exception, where additional iterations do not improve on the zero-shot solver.

These results suggest that the synthesis loop does more than sample alternative implementations: later rounds can revise an initially poor specialization or turn a useful structural idea into a substantially faster solver. The benefit is not universal, but on several problem families it is a major source of the final runtime gain.

{\bf Hint ablation.\enspace}
To isolate explicit structure discovery from iterative code refinement, we compare against a matched GPT-5.2 pipeline with the same beam, candidate budget, validation feedback, refinement, and repair, but no hypothesis/hint stage. The full pipeline matches or exceeds this baseline on all seven representative targets and improves mean quality by \(0.040\), with the largest gains on Ring-template Coloring (\(+0.203\)) and Latent-class MDKP (\(+0.066\)). Full results appear in Appendix~\ref{app:no_hint_ablation}.

{\bf What did the LLM compile?\enspace}
The speedups in Table~\ref{tab:headline_results} are not just implementation
effects. In many cases, the selected solver changes the effective computation:
it replaces an ambient worst-case search or generic optimization call with a
distribution-specific procedure inferred from samples. Table~\ref{tab:compiled_structure_summary}
gives a compact summary of these computation patterns, while Appendix~\ref{app:distributional_complexity} (specifically Table~\ref{tab:dac_all_distributions}) gives the full per-distribution
expressions and notation.

\begin{table}[t]
\centering
\caption{
{\bf Representative computation patterns.}
Ambient denotes generic computation; generated denotes the
distribution-specific computation synthesized from samples.
Full expressions and notation appear in Table~\ref{tab:dac_all_distributions}.
}
\label{tab:compiled_structure_summary}

\scriptsize
\setlength{\tabcolsep}{2.6pt}
\renewcommand{\arraystretch}{1.14}
\definecolor{ambientcol}{RGB}{248,248,248}

\begin{adjustbox}{max width=\linewidth,center}
\begin{tabular}{@{}
l l c l
!{\hspace{5pt}\vrule width 0.4pt\hspace{5pt}}
l l c l
@{}}
\toprule

\textbf{Problem}
& \textbf{Structure}
& \cellcolor{white}\textbf{Ambient}
& \textbf{Generated}
&
\textbf{Problem}
& \textbf{Structure}
& \cellcolor{white}\textbf{Ambient}
& \textbf{Generated}
\\

\midrule

Coloring
& ring template
& \cellcolor{ambientcol}$O^*(\kappa^n)$
& $O(m)$ fast path
&
MDS
& star kernel
& \cellcolor{ambientcol}$O^*(2^n)$
& $O(n+m+\sum_j c_j2^{t_j}+T_{\rm prune})$
\\

MAXSAT
& latent rules
& \cellcolor{ambientcol}$O^*(2^v)$
& $O(|F|\ell+RB\ell\Delta_{\rm occ})$
&
Packing LP
& bottleneck resource
& \cellcolor{ambientcol}$T_{\rm LP}(N,r,L)$
& $O(Nr+N\log N)$
\\

MIS
& core--fringe
& \cellcolor{ambientcol}$O^*(2^n)$
& $O(n+m+\sum_j4^{t_j}+kq)$
&
MDKP
& single resource
& \cellcolor{ambientcol}$O^*(2^N)$
& $O(Nr+N\log N+K_{\rm cand}C_b+T_{\rm repair})$
\\

TSP
& paired ribbon
& \cellcolor{ambientcol}$O(n^2 2^n)$
& $O(n\log n)$ fast path
&
\multicolumn{4}{c}{}
\\

\bottomrule
\end{tabular}
\end{adjustbox}

\end{table}

\begin{table}[t]
\centering
\caption{
\textbf{PACE 2025 Dominating Set and Hitting Set.}
On both released private benchmarks, our solver is valid on all \(100\) instances and occupies a low-runtime operating point, trading modestly worse objective values for substantially faster deployment than the released PACE solvers.
}
\label{tab:pace_main}

\scriptsize
\renewcommand{\arraystretch}{1.05}

\begin{minipage}[t]{0.49\linewidth}
\centering
\textbf{Dominating Set}\\[2pt]
\setlength{\tabcolsep}{2.3pt}
\begin{adjustbox}{max width=\linewidth}
\begin{tabular}{@{}lccccc@{}}
\toprule
\textbf{Solver}
& \textbf{Valid}
& \textbf{Avg.\ size $\downarrow$}
& \textbf{Ratio}
& \textbf{Time $\downarrow$}
& \textbf{Speedup} \\
\midrule

\rowcolor{ourpurple}
Ours (GPT-5.2)
& \(100/100\)
& \(231{,}595\)
& \(1.00\)
& \(2.89\)s
& \(1.0\times\) \\

\rowcolor{ourpurple}
Ours (Gemma 4)
& \(100/100\)
& \(231{,}667\)
& \(1.00\)
& \(6.03\)s
& \(2.1\times\) \\

\midrule
Gemma 4 best@5
& \(100/100\)
& \(261{,}741\)
& \(1.13\)
& \(2.69\)s
& \(0.93\times\) \\

Codex 1-shot
& \(100/100\)
& \(230{,}863\)
& \(0.997\)
& \(2.43\)s
& \(0.84\times\) \\

Claude Code 1-shot
& \(100/100\)
& \(224{,}699\)
& \(0.970\)
& \(37.29\)s
& \(12.9\times\) \\

\midrule
Released PACE solvers
& \(75\)--\(100/100\)
& \(210\text{k}\)--\(224\text{k}\)
& \(0.91\)--\(0.97\)
& \(218\)--\(360\)s
& \(75\)--\(125\times\) \\

General heuristics
& \(100/100\)
& \(235\text{k}\)--\(242\text{k}\)
& \(1.01\)--\(1.04\)
& \(8.3\)s
& \(2.9\times\) \\

Gurobi / exact-style
& timeout
& ---
& ---
& \(\ge 360\)s
& \(\ge 124.6\times\) \\

\bottomrule
\end{tabular}
\end{adjustbox}
\end{minipage}
\hfill
\begin{minipage}[t]{0.49\linewidth}
\centering
\textbf{Hitting Set}\\[2pt]
\setlength{\tabcolsep}{2.3pt}
\begin{adjustbox}{max width=\linewidth}
\begin{tabular}{@{}lcccc@{}}
\toprule
\textbf{Solver}
& \textbf{Valid}
& \textbf{Avg.\ size $\downarrow$}
& \textbf{Time $\downarrow$}
& \textbf{Speedup} \\
\midrule

\rowcolor{ourpurple}
Ours (GPT-5.2)
& \(100/100\)
& \(197{,}477\)
& \(3.75\)s
& \(1.0\times\) \\

Claude Code 1-shot
& \(100/100\)
& \(201{,}614\)
& \(23.36\)s
& \(6.23\times\) \\

Codex 1-shot
& \(100/100\)
& \(195{,}439\)
& \(10.56\)s
& \(2.82\times\) \\

Gurobi (420 s)
& \(100/100\)
& \(195{,}596\)
& \(433.7\)s
& \(115.7\times\) \\

\midrule
Released PACE solvers
& \(100/100\)
& \(191.5\text{k}\)--\(192.2\text{k}\)
& \(300.7\)--\(360.3\)s
& \(80.2\)--\(96.1\times\) \\

\midrule
frequency\_greedy
& \(100/100\)
& \(205{,}938\)
& \(6.12\)s
& \(1.63\times\) \\

exact (catalog)
& \(0/100\)
& ---
& \(0.62\)s
& --- \\

\bottomrule
\end{tabular}
\end{adjustbox}
\end{minipage}

\vspace{1pt}
{\tiny
For Dominating Set, Swats statistics use the \(75/100\) instances on which it
returns a valid solution. The Hitting Set exact-catalog method returns no valid
solutions, so its runtime is not a successful-solve comparison.
}
\end{table}

\begin{table}[t]
\centering
\caption{
\textbf{PACE 2024 One-Sided Crossing Minimization.}
On both released private tracks, our solver remains close in objective value to the fully valid PACE solvers while running substantially faster, including more than \(34{,}000\times\) on the exact track. Parentheses give objective ratios relative to ours; speedup is baseline runtime divided by ours.
}
\label{tab:pace_ocm}

\scriptsize
\setlength{\tabcolsep}{2.5pt}
\renewcommand{\arraystretch}{1.05}

\begin{minipage}[t]{0.49\linewidth}
\centering
\textbf{Cutwidth track}\\[2pt]

\begin{adjustbox}{max width=\linewidth}
\begin{tabular}{@{}lcccc@{}}
\toprule
\textbf{Solver}
& \textbf{Valid}
& \textbf{Obj. $\downarrow$\ (ratio)}
& \textbf{Time $\downarrow$}
& \textbf{Speedup} \\
\midrule

\rowcolor{ourpurple}
Ours (GPT-5.2)
& \(100/100\)
& \(101{,}362\;(1.000)\)
& \(0.016\)
& \(1.0\times\) \\

Gemma 4 best@5
& \(100/100\)
& \(101{,}396\;(1.000)\)
& \(0.014\)
& \(0.85\times\) \\

Claude Code 1-shot
& \(100/100\)
& \(101{,}209\;(0.998)\)
& \(1.193\)
& \(74.1\times\) \\

Codex 1-shot
& \(100/100\)
& \(101{,}233\;(0.999)\)
& \(0.402\)
& \(25.0\times\) \\

ocmu64
& \(100/100\)
& \(101{,}209\;(0.998)\)
& \(0.058\)
& \(3.6\times\) \\

pingpong
& \(100/100\)
& \(101{,}209\;(0.998)\)
& \(0.179\)
& \(11.1\times\) \\

cimat
& \(0/100\)
& ---
& \(1.129\)
& \(70.0\times\) \\

\bottomrule
\end{tabular}
\end{adjustbox}
\end{minipage}
\hfill
\begin{minipage}[t]{0.49\linewidth}
\centering
\textbf{Exact track}\\[2pt]

\begin{adjustbox}{max width=\linewidth}
\begin{tabular}{@{}lcccc@{}}
\toprule
\textbf{Solver}
& \textbf{Valid}
& \textbf{Obj. $\downarrow$\ (ratio)}
& \textbf{Time $\downarrow$}
& \textbf{Speedup} \\
\midrule

\rowcolor{ourpurple}
Ours (GPT-5.2)
& \(100/100\)
& \(148{,}091\;(1.000)\)
& \(0.0085\)
& \(1.0\times\) \\

Gemma 4 best@5
& \(100/100\)
& \(159{,}387\;(1.076)\)
& \(0.0055\)
& \(0.65\times\) \\

Claude Code 1-shot
& \(100/100\)
& \(143{,}207\;(0.967)\)
& \(1.235\)
& \(145.3\times\) \\

Codex 1-shot
& \(100/100\)
& \(143{,}260\;(0.967)\)
& \(0.234\)
& \(27.5\times\) \\

cimat
& \(100/100\)
& \(143{,}195\;(0.967)\)
& \(295.039\)
& \(34{,}604\times\) \\

pingpong\(^{a}\)
& \(97/100\)
& \(143{,}978\;(0.972)\)
& \(24.734\)
& \(2{,}901\times\) \\

ocmu64\(^{a}\)
& \(78/100\)
& \(118{,}547\;(0.800)\)
& \(86.764\)
& \(10{,}176\times\) \\

\bottomrule
\end{tabular}
\end{adjustbox}
\end{minipage}

\vspace{2pt}
{\tiny
Time is in seconds. Lower objective is better; an objective ratio below \(1\)
means the baseline has a lower objective than GPT-5.2.
\(^{a}\)Objectives are averaged only over instances with valid solutions and are
therefore not directly comparable with \(100/100\) rows.
}

\end{table}

The main pattern is that synthesis often compiles a structural shortcut. For
MAXSAT, generated solvers exploit latent Boolean rules and bounded local repair
instead of searching over all assignments. For graph problems, they exploit
palettes, motifs, hubs, gateways, or small residual kernels rather than treating
each instance as an arbitrary coloring, MIS, or MDS problem. For Packing LP and
MDKP, they exploit active resources or bottlenecks, replacing generic
optimization with sorting, scoring, fractional filling, and bounded repair. For
TSP, they exploit recurring geometry to construct and improve tours without
full Held--Karp-style dynamic programming.

Thus, the empirical speedups reflect the intended mechanism of distribution-aware programming: the LLM is not merely producing faster code for the same computation, but often compiling sampled regularities into a lower-dimensional or smaller-residual computation.

{\bf External PACE benchmarks.\enspace}
We next test whether the synthesis procedure transfers beyond our controlled
distributions to independently constructed competition instances. We evaluate
on the released PACE 2025 Dominating Set and Hitting Set heuristic tracks and
the PACE 2024 One-Sided Crossing Minimization (OCM) tracks. In all cases,
synthesis uses only released public instances, while released private instances
are held out for evaluation.

On PACE 2025 Dominating Set, our solver returns valid solutions on all \(100\)
private instances in \(2.89\) seconds on average. Released competition solvers
produce sets about \(3\%\) smaller but require \(75\)--\(125\times\) more time.
The Codex-synthesized solver is slightly better and faster than ours, while the
Claude Code-synthesized solver is better but substantially slower. Hitting Set
shows a similar tradeoff: our solver is valid on all \(100\) private instances
and averages \(3.75\) seconds, while released solvers produce sets about \(3\%\)
smaller but require \(80\)--\(96\times\) more time
(Table~\ref{tab:pace_main}). On Dominating Set, none of three released anytime
solvers matches our all-instance quality at \(3\)s or \(10\)s, although all
three surpass it at \(30\)s (Appendix~\ref{app:pace_ds}).

On PACE 2024 OCM, our solver is within about \(0.2\%\) of the fully valid
released solvers on the cutwidth track while running \(3.6\)--\(11.1\times\)
faster. On the exact track, it returns valid solutions on all \(100\) private
instances and is within \(3.4\%\) of the best fully valid released solver while
running more than \(34{,}000\times\) faster
(Table~\ref{tab:pace_ocm}). Across the four PACE tasks, the one-shot
Codex- and Claude Code-synthesized solvers attain slightly higher normalized quality, while
ours is much faster on three; the exception is the Codex-synthesized Dominating Set solver.
Together, these results test the same sample-to-solver procedure on external
distributions constructed independently of our benchmark suite and against
solvers developed specifically for the corresponding competitions.

\section{Discussion and Limitations}
\label{sec:conclusion}

The broader implication is that repeated computation can become progressively cheaper with experience. Rather than solving each instance from scratch, AI can learn how a recurring workload should be solved and compile that knowledge into executable algorithms. A natural next step is to make this specialization adaptive as workloads evolve.

Several limitations remain. Synthesis cost must be amortized over future instances, specialization may degrade under distribution shift, and the method does not dominate every baseline. Across repeated synthesis runs, quality is relatively stable but runtime varies more (Appendix~\ref{app:synthesis_variance}), suggesting that program search can yield similarly accurate but computationally different solvers. Robustness to distribution shift and synthesis variability remains important future work.

\clearpage

\bibliographystyle{iclr2027_conference}
\bibliography{refs}

\clearpage

\appendix

\section{Broader Impacts}
\label{app:broader_impacts}

This work studies methods for learning distribution-specific solvers from examples.
Potential positive impacts include reducing the computational cost of repeated
optimization workloads, making specialized solvers easier to construct, and
improving access to efficient algorithmic tools in scientific, engineering, and
logistics settings. Potential negative impacts arise if synthesized solvers are
deployed outside the distributions for which they were inferred, where brittle
specialization may lead to incorrect or inefficient decisions. More broadly,
automated solver synthesis could also reduce the barrier to optimizing objectives
in harmful or poorly governed applications. We therefore view validation on
held-out instances, explicit distributional assumptions, fallback mechanisms, and
human review of generated code as important safeguards.




\section{Proof of Prop.~\ref{prop:sample_conditioned_complexity}}

\propsampleconditionedcomplexity*

\begin{proof}
Fix $D$, and let $\widehat A_S$ minimize the empirical runtime
$\widehat R_S(A)=m^{-1}\sum_{i=1}^m T(A,X_i)$ over the base programs.
Hoeffding's lemma gives
$\mathbb E_S\exp\{\lambda(R_D(A)-\widehat R_S(A))\}
\le \exp\{\lambda^2B^2/(8m)\}$.
Applying the standard exponential-moment bound to the maximum over $N$
programs yields
$\mathbb E_S\max_A(R_D(A)-\widehat R_S(A))
\le B\sqrt{\log N/(2m)}$.

Let $A_D^\star$ minimize $R_D(A)$. Since
$\widehat R_S(\widehat A_S)\le\widehat R_S(A_D^\star)$, taking expectations
gives
\[
\mathbb E_S[R_D(\widehat A_S)]
\le
\inf_{A\in\mathcal A}R_D(A)
+
B\sqrt{\frac{\log N}{2m}}.
\]
Taking the supremum over $D$ and using
$\log N\le(\ell+1)\log 2$ gives the stated upper bound.
A learner may always ignore its sample, so $C_m\le C_0$. Conversely, for
each fixed $D$, no learner can achieve expected deployment runtime below
$\inf_A R_D(A)$, which gives $C_{\mathrm{known}}\le C_m$.
\end{proof}

\section{Additional Method Details}
\label{app:method_details}

This appendix gives the protocol details omitted from Section~\ref{sec:method}: what information the synthesis agent receives, how candidates are represented and selected, and how failures are handled. The analysis program may use the public training sample to compute a reusable summary, but the deployed solver receives only that summary and the new instance.

Each LLM proposal is treated as a candidate rather than the learned algorithm itself. The search produces alternative structural hypotheses and solver implementations, some of which may be slow, brittle, or incorrect. Candidates are executed and ranked using public validation quality, optimality, and runtime, and the selected candidate is deployed.

{\bf Public information and leakage boundary.\enspace}
For each target, the agent receives the public training and validation instances and a task manifest specifying the problem, output format, quality metric, and basic instance sizes. It is also told that the instances come from a shared but unknown structured distribution.

The agent does not see the distribution family, generative rule, optimum solutions or values, test instances, or evaluator-only metadata. Thus it knows the optimization problem and how solutions are scored, but not the mechanism generating the instances.

All synthesis stages share the same objective: infer structure from the public samples and use it to produce a fast solver. Validation quality is the primary selection criterion, with runtime used to distinguish high-quality candidates.

{\bf Candidate representation.\enspace}
Each candidate is \(c=(H_c,A_c,a_c,s_c)\). Here \(H_c\) is a hypothesis about recurring structure, \(A_c\) is an analysis program, \(a_c=A_c(S_{\mathrm{tr}}^{\mathrm{pub}})\) is the resulting solver hint, and \(s_c\) is the deployment solver. Given a new instance \(x^{\mathrm{pub}}\), the solver returns
\(z=s_c(x^{\mathrm{pub}},a_c)\).

Table~\ref{tab:prompt_schemas} summarizes what is produced at each stage.

\begin{table}[t]
\centering
\caption{
{\bf Synthesis-stage outputs.}
The diversity key is used for beam selection, not validation scoring.
}
\label{tab:prompt_schemas}
\small
\setlength{\tabcolsep}{6pt}
\renewcommand{\arraystretch}{1.15}

\begin{tabularx}{\linewidth}{
>{\raggedright\arraybackslash}p{0.16\linewidth}
>{\raggedright\arraybackslash}X
}
\toprule
\rowcolor{headergray}
\textbf{Stage} & \textbf{Output} \\
\midrule

\rowcolor{hypbg}
\textcolor{hypfg}{\textbf{Hypothesis}} &
Suspected reusable structure; evidence to measure; the corresponding solver
strategy; likely failure modes; and a diversity key identifying the broad
hypothesis class. \\[0.3em]

\rowcolor{anabg}
\textcolor{anafg}{\textbf{Analysis}} &
A program applied to \(S_{\mathrm{tr}}^{\mathrm{pub}}\) that estimates the
relevant structure and returns a compact reusable summary rather than the
training instances themselves. \\[0.3em]

\rowcolor{solbg}
\textcolor{solfg}{\textbf{Solver}} &
A deployment program for new instances that uses the recovered summary, with
fallback behavior for weak or ambiguous structure and without repeating the
training-time analysis. \\
\bottomrule
\end{tabularx}

\end{table}

{\bf Synthesis loop.\enspace}
Algorithm~\ref{alg:llm_synthesis} summarizes the search procedure. In each
round, the agent generates candidate hypotheses, analysis programs, and solvers;
the resulting solvers are evaluated on the public training and validation sets,
and the beam is updated. We keep track of all evaluated candidates and deploy
the best one found over the full search.

\begin{algorithm}[t]
\caption{LLM-based solver synthesis}
\label{alg:llm_synthesis}
\begin{algorithmic}[1]

\Require Public train/validation sets
\(S_{\mathrm{tr}}^{\mathrm{pub}},S_{\mathrm{val}}^{\mathrm{pub}}\);
rounds \(R\), beam width \(B\), candidate budget \(K\)

\Ensure Solver \deploybox{\(\hat s\)} and summary \hintbox{\(\hat a\)}

\State Initialize beam \(\mathcal B_0\) with diverse hypotheses
\State \(\mathcal E\gets\emptyset\)

\For{\(r=0,\ldots,R-1\)}

    \State Form \(K\) candidate prompts from \(\mathcal B_r\)

    \For{each candidate prompt}
        \State \hypbox{Hypothesis}
        \(\textcolor{hypfg}{H_c}\)

        \State \anabox{Analysis}
        \(\textcolor{anafg}{A_c}\)

        \State \hintbox{Run analysis}
        \(\textcolor{anafg}{a_c \gets A_c(S_{\mathrm{tr}}^{\mathrm{pub}})}\)

        \State \solbox{Solver}
        \(\textcolor{solfg}{s_c}\gets
        \mathrm{LLM}(H_c,A_c,a_c)\)

        \State \evalbox{Evaluate}
        \(s_c(\cdot,a_c)\) on
        \(S_{\mathrm{tr}}^{\mathrm{pub}}\) and
        \(S_{\mathrm{val}}^{\mathrm{pub}}\)

        \State Add
        \(c=(\textcolor{hypfg}{H_c},
        \textcolor{anafg}{A_c},
        \textcolor{anafg}{a_c},
        \textcolor{solfg}{s_c})\)
        to \(\mathcal E\)
    \EndFor

    \State Rank candidates by
    \[
    \bigl(Q_{\mathrm{val}}(c),\,
          O_{\mathrm{val}}(c),\,
          -T_{\mathrm{val}}(c)\bigr)
    \]

    \State Set \(\mathcal B_{r+1}\) by keeping the best candidate
    from each diversity class, then filling remaining slots by score

\EndFor

\State
\(c^\star\gets\arg\max_{c\in\mathcal E}\mathrm{Score}(c)\)

\State
\(\textcolor{anafg}{\hat a\gets
A_{c^\star}(S_{\mathrm{tr}}^{\mathrm{pub}})}\)

\State \Return
\deploybox{\(\hat s:=s_{c^\star}\)}
and
\hintbox{\(\hat a\)}

\end{algorithmic}
\end{algorithm}

{\bf Beam selection and refinement.\enspace}
Each hypothesis is assigned a broad class, such as separator-based, bottleneck-resource, latent-cluster, or geometric-template. After each round, candidates are ranked by the validation score in Algorithm~\ref{alg:llm_synthesis}. We retain the best candidate from each class and fill the remaining beam slots by score.

Refinement uses the parent hypothesis, recovered summary, train and validation metrics, and representative failures. The agent may revise the current hypothesis, branch to another explanation, or improve quality or runtime. Hidden distribution metadata and test results are never provided. The diversity rule simply keeps several structural explanations in play while validation determines which candidates survive.

{\bf Recovered summaries.\enspace}
The summary \(a_c\) has no fixed representation. Depending on the problem, it may contain vertex-role scores, clusters, separators, Boolean statistics, resource bottlenecks, item classes, or geometric structure. It is not scored directly; only the quality and runtime of the resulting solver matter.

The analysis program computes \(a_c\) once from the training sample, and the deployed solver reuses it on future instances. Thus distribution-level inference is performed before deployment rather than repeated for every solve.

{\bf Failure handling.\enspace}
Candidates may fail during analysis, solver construction, or execution. If \(A_c\) fails or exceeds its time budget, or if \(s_c\) cannot be loaded, the candidate receives zero quality and a large failure runtime. If a running solver produces an invalid output on only some instances, those instances receive zero quality and the rest are scored normally. Failed candidates are not manually repaired.

{\bf Connection to the analysis.\enspace}
The protocol is a practical version of the sample-conditioned selection problem in Section~\ref{sec:analysis}. The hypothesis \(H_c\) proposes reusable structure, \(A_c\) estimates it from samples, and \(s_c\) turns the resulting summary into a deployment algorithm. Unlike the idealized analysis, neither the relevant structure nor its representation is fixed in advance: both are proposed by the LLM and selected through validation. The search therefore approximates a sample-to-solver map by generating and testing candidate specializations. 

\section{Additional Experimental Details}
\label{app:additional_details}

\subsection{Benchmark Distributions}
\label{app:benchmark_distributions}

\begin{table}[t]
\centering
\caption{
\textbf{Benchmark target distributions.}
Seven problem classes, each with three hidden distribution families, giving
$21$ targets. Each target is split into $64$ training, $32$ validation, and
$500$ held-out test instances.
}
\label{tab:benchmark_targets}
\small
\setlength{\tabcolsep}{5pt}
\renewcommand{\arraystretch}{1.1}

\newcommand{\tgt}[2]{#1\newline{\scriptsize\textcolor{black!55}{#2}}}
\newcommand{\prb}[2]{\textbf{#1}\newline{\scriptsize\itshape\textcolor{black!55}{#2}}}

\begin{adjustbox}{max width=\linewidth,center}
\begin{tabular}{@{}
    >{\raggedright\arraybackslash}p{0.155\linewidth}
    *{3}{>{\raggedright\arraybackslash}p{0.235\linewidth}}@{}}
\toprule
\textbf{Problem} \newline {\scriptsize\itshape objective}
& \multicolumn{3}{c}{\textbf{Hidden distribution families} \ {\scriptsize\itshape (name, instance size)}} \\
\cmidrule(l){2-4}
\addlinespace[2pt]

\prb{Coloring}{min.\ colors}
& \tgt{Ring-template}{169 vertices}
& \tgt{Overlapping-palette}{340 vertices}
& \tgt{Separator-trap}{238 vertices} \\
\addlinespace[5pt]

\prb{MAXSAT}{max.\ sat.\ clauses}
& \tgt{Community-parity}{240 var.\ / 960 cl.}
& \tgt{Last-clause signal}{280 var.\ / 1120 cl.}
& \tgt{Latent-backdoor}{128 var.\ / 512 cl.} \\
\addlinespace[5pt]

\prb{MIS}{max.\ set size}
& \tgt{Clique-path}{190 vertices}
& \tgt{Core-fringe trap}{1000 vertices}
& \tgt{Motif-bridge}{195 vertices} \\
\addlinespace[5pt]

\prb{MDS}{min.\ set size}
& \tgt{Gateway-hub}{2800 vertices}
& \tgt{Geometric-anchor}{1600 vertices}
& \tgt{Star-kernel}{2800 vertices} \\
\addlinespace[5pt]

\prb{Packing LP}{max.\ value}
& \tgt{Block-coupled}{1200 items / 40 res.}
& \tgt{Active-resource}{1200 items / 40 res.}
& \tgt{Single-bottleneck}{1200 items / 40 res.} \\
\addlinespace[5pt]

\prb{MDKP}{max.\ value}
& \tgt{Decoy-complement}{1040 items / 48 res.}
& \tgt{Latent-class}{520 items / 32 res.}
& \tgt{Single-resource}{1040 items / 48 res.} \\
\addlinespace[5pt]

\prb{TSP}{min.\ tour length}
& \tgt{Clustered Euclidean}{120 cities}
& \tgt{Latent-metric}{120 cities}
& \tgt{Paired-ribbon zigzag}{320 cities} \\
\addlinespace[2pt]

\bottomrule
\end{tabular}
\end{adjustbox}
\end{table}

Table~\ref{tab:benchmark_targets} summarizes the \(21\) benchmark targets: seven problem classes, each with three hidden distribution families. Every target is split into \(64\) training, \(32\) validation, and \(500\) held-out test instances. The distributions are structured rather than adversarial, with recurring regularities that a solver may exploit after seeing a small sample. These include planted assignments, latent classes, bottlenecks, geometric layouts, motifs, and hidden resource constraints.

{\bf Graph coloring.\enspace}
\emph{Ring-template} partitions vertices into local blocks with compatible palettes and ring-like bridges. \emph{Overlapping-palette} uses blocks with shifted, overlapping palettes, so local coloring alone is insufficient. \emph{Separator-trap} introduces separator vertices that create long-range color-reuse constraints.

{\bf MAXSAT.\enspace}
\emph{Community-parity} partitions variables into communities controlled by parity-style functions of a few anchor bits. In \emph{Last-clause signal}, the final clause reveals an anchor pattern and the remaining variables copy or negate anchor bits according to a hidden rule table. \emph{Latent-backdoor} mixes several regimes in which variable blocks follow anchor-dependent Boolean rules, with additional noise and bridge clauses.

{\bf Maximum independent set.\enspace}
\emph{Clique-path} alternates clique-like and path-like blocks connected by sparse bridges. \emph{Core-fringe trap} places a low-degree fringe around a dense core, making naive low-degree greedy rules misleading. \emph{Motif-bridge} draws instances from latent sequences of cliques, cycles, bicliques, crown-like gadgets, and sparse inter-motif bridges.

{\bf Minimum dominating set.\enspace}
\emph{Gateway-hub} contains clusters with hubs and gateways whose coverage overlaps across neighboring clusters. \emph{Geometric-anchor} uses geometric clusters of varying density connected by sparse edges. \emph{Star-kernel} consists of clusters with stable local hubs and sparse inter-cluster connectors.

{\bf Packing LP.\enspace}
\emph{Block-coupled} groups items into latent resource blocks linked by a shared coupling constraint. \emph{Active-resource} varies the active resource regime across instances, producing recurring dual-price patterns. In \emph{Single-bottleneck}, one resource is consistently tighter than the others, making value per unit of that resource highly informative.

{\bf Multidimensional knapsack.\enspace}
\emph{Decoy-complement} includes high-value decoys that consume a scarce hidden resource, while complementary item classes combine more efficiently. \emph{Latent-class} draws items from hidden resource-consumption classes and varies the active bottleneck across instances. In \emph{Single-resource}, one resource is usually dominant, making density with respect to it a strong but imperfect guide.

{\bf Traveling salesperson problem.\enspace}
\emph{Clustered Euclidean} samples cities from balanced clusters arranged around a ring, favoring tours with short intra-cluster paths and the correct inter-cluster order. \emph{Latent-metric} mixes several geometric regimes, including ring clusters, paired ribbons, and barrier-bridge layouts. In \emph{Paired-ribbon zigzag}, cities lie on two nearly parallel rails, so the best tours traverse one rail and return along the other. 

\subsection{Evaluation Metrics}
\label{app:evaluation_metrics}

We evaluate solvers by normalized solution quality, optimality rate, and runtime. Because the benchmark mixes maximization and minimization problems, we map objective values to a common quality score in \([0,1]\), with \(1\) denoting an optimum. Table~\ref{tab:quality_metrics} gives the normalization for each problem class.

\begin{table}[t]
\centering
\caption{
{\bf Normalized solution quality.}
Quality is the ratio between the returned and optimal objective, oriented so
that higher is always better: optimal solutions have quality \(1\), and invalid
or infeasible outputs receive \(0\).
}
\label{tab:quality_metrics}
\small
\setlength{\tabcolsep}{6pt}
\renewcommand{\arraystretch}{1.15}

\newcommand{\qgroup}[2]{%
  \addlinespace[2pt]%
  \multicolumn{3}{@{}l}{\itshape\textcolor{black!60}{#1}\ \ \(#2\)}\\
  \addlinespace[2pt]}

\begin{adjustbox}{max width=\linewidth,center}
\begin{tabular}{@{}l l l@{}}
\toprule
\textbf{Problem} & \textbf{Measured quantity} & \textbf{Normalized quality} \\
\midrule

\qgroup{Maximization}{\text{quality}=\text{returned}/\text{optimum}}

MAXSAT
& satisfied clauses \(\mathrm{sat}\)
& \(\mathrm{sat}_{\rm alg}/\mathrm{sat}_{\rm opt}\) \\

MIS
& independent-set size \(|\mathrm{IS}|\)
& \(|\mathrm{IS}_{\rm alg}|/|\mathrm{IS}_{\rm opt}|\) \\

Packing LP, MDKP
& objective value \(\mathrm{obj}\)
& \(\mathrm{obj}_{\rm alg}/\mathrm{obj}_{\rm opt}\) \\

\midrule
\qgroup{Minimization}{\text{quality}=\text{optimum}/\text{returned}}

Coloring
& colors used \(k\)
& \(k_{\rm opt}/k_{\rm alg}\) \\

MDS
& dominating-set size \(|\mathrm{DS}|\)
& \(|\mathrm{DS}_{\rm opt}|/|\mathrm{DS}_{\rm alg}|\) \\

TSP
& tour length \(\mathrm{len}\)
& \(\mathrm{len}_{\rm opt}/\mathrm{len}_{\rm alg}\) \\

\bottomrule
\end{tabular}
\end{adjustbox}
\end{table}

{\bf Normalized quality.\enspace}
For an instance \(x\) and solver \(A\), define
\[
q_A(x)=
\begin{cases}
\mathrm{val}_A(x)/\mathrm{val}_{\rm opt}(x),
& \text{for maximization},\\[2pt]
\mathrm{cost}_{\rm opt}(x)/\mathrm{cost}_A(x),
& \text{for minimization},
\end{cases}
\]
when the returned solution is feasible. Invalid, malformed, or infeasible outputs receive \(q_A(x)=0\). All benchmark optima are positive, so these ratios are well defined.

{\bf Optimality rate.\enspace}
The optimality rate is the fraction of held-out instances with \(q_A(x)=1\). Equality is checked by the benchmark verifier, with numerical tolerance for Packing LP. Invalid or infeasible outputs are not counted as optimal.

{\bf Runtime.\enspace}
Runtime is measured per instance in milliseconds and includes all deployment-time computation, including verification, repair, fallback, and local search. Offline synthesis is excluded.

{\bf Aggregation.\enspace}
Each of the \(21\) target distributions receives equal weight. We first average quality, optimality, feasibility, and runtime over held-out instances and repeated runs within each target. Additive quantities, including quality differences
\(\Delta Q_b=Q_{\rm ours}-Q_b\), are then averaged across targets. Runtime comparisons use $S_b(\tau)=\frac{T_b(\tau)}{T_{\rm ours}(\tau)}$, so \(S_b>1\) means our solver is faster; reported speedups are geometric means over targets. Comparator definitions are given in Section~\ref{sec:experiments}.

\subsection{Baseline Catalog and Reporting Protocol}
\label{app:baseline_catalog_protocol}

The main text summarizes the baseline groups; this appendix gives the full per-problem catalog and execution details. Table~\ref{tab:baseline_catalog} lists \(78\) registered entries across seven problem classes, corresponding to \(72\) unique implementations after removing aliases. We retain aliases because they are executable registry entries; for example, several \texttt{exact} entries point to problem-specific implementations such as OR-Tools formulations or Held--Karp dynamic programming.

Classical baselines fall into three groups. \emph{Heuristics} are local problem-specific methods such as greedy, local-search, insertion, rounding, and randomized rules. \emph{Solver-backed} methods call general-purpose optimization software under the benchmark protocol. Gurobi is reported separately with a \(10\)-second single-thread budget and may return an incumbent without proving optimality. \emph{Exact} and \emph{external exact} methods provide certified comparisons when they terminate with an optimum.

Each problem class also has one architecture-matched learned baseline rather than a single shared model: GNN-based methods for coloring, MIS, and MDS; RUN-CSP-style models for MaxSAT; an attention model for TSP; and learned packing methods for MDKP and Packing LP. The LLM baselines are one-shot Codex, one-shot Claude Code, and best-of-\(5\) Gemma~4; these are described in the main text.

In aggregate tables, \(\mathrm{Heur}\) denotes the fastest heuristic within \(95\%\) of the best heuristic quality for that class, \(\mathrm{avg}\) the average heuristic, \(\mathrm{Gur}\) Gurobi, and \(\mathrm{Exact}\) the fastest certified exact method attaining the best validated quality. All baselines use the same public-instance protocol and receive no information about the hidden distribution family or generative rule.

\begin{table}[t]
\centering
\caption{
{\bf Classical and learned baseline catalog.}
The top table lists problem-specific heuristics, exact or certifying methods,
and the learned baseline for each problem class. The bottom table lists
general-purpose backends shared across classes. Each class also has an
\texttt{exact} alias for its designated exact backend.
LLM baselines are described in the main text.
}
\label{tab:baseline_catalog}

\scriptsize
\setlength{\tabcolsep}{4pt}
\renewcommand{\arraystretch}{1.12}

\begin{adjustbox}{max width=\linewidth,center}
\begin{tabular}{@{}
    >{\raggedright\arraybackslash}p{0.105\linewidth}
    >{\raggedright\arraybackslash}p{0.365\linewidth}
    >{\raggedright\arraybackslash}p{0.325\linewidth}
    >{\raggedright\arraybackslash}p{0.145\linewidth}@{}}
\toprule
\textbf{Problem}
& \textbf{Heuristics}
& \textbf{Exact / certifying}
& \textbf{Learned} \\
\midrule

Coloring
& DSATUR and greedy variants: largest-degree, random-order, and smallest-last
  orderings~\citep{brelaz1979new}
& DSATUR branch-and-bound~\citep{brelaz1979new,land1960automatic}
& PI-GNN~\citep{schuetz2022graphcoloring,schuetz2022combinatorial} \\
\addlinespace[3pt]

MaxSAT
& Greedy flip local search; literal-majority polarity; random assignment
& EvalMaxSAT~\citep{avellaneda2020evalmaxsat};
  MaxHS~\citep{davies2013exploiting,davies2013postponing};
  Open-WBO~\citep{martins2014openwbo};
  UWrMaxSAT~\citep{piotrow2020uwrmaxsat};
  WMaxCDCL~\citep{coll2025wmaxcdcl}
& RUN-CSP~\citep{toenshoff2021runcsp} \\
\addlinespace[3pt]

MIS
& Local improvement; minimum-degree, random, and ratio-based greedy
& Clique branch-and-bound~\citep{land1960automatic};
  KaMIS via exact vertex cover~\citep{lamm2019exactly}
& PI-GNN~\citep{schuetz2022combinatorial} \\
\addlinespace[3pt]

MDS
& High-degree, marginal-gain, and redundancy-aware greedy~\citep{chvatal1979greedy}
& Set-cover branch-and-bound~\citep{chvatal1979greedy,land1960automatic}
& GNN${+}$RL~\citep{chen2024dominatingset} \\
\addlinespace[3pt]

Packing LP
& Density-based fractional packing; uniform-fraction LP~\citep{dantzig1957discrete}
& Generic LP backends only
& Primal--dual net~\citep{park2023selfsupervised} \\
\addlinespace[3pt]

MDKP
& LP-relaxation rounding; redundancy-improved greedy; value-density
  greedy~\citep{dantzig1957discrete}
& Custom branch-and-bound~\citep{land1960automatic}
& Deep RL agent~\citep{bushaj2024kmeans} \\
\addlinespace[3pt]

Euclidean TSP
& Random tour; nearest neighbor; nearest and farthest insertion; multi-start
  two-opt; nearest-neighbor and farthest-insertion plus
  two-opt~\citep{lin1965computer,rosenkrantz1977analysis};
  LKH~\citep{helsgaun2000effective}
& Concorde~\citep{applegate2006traveling};
  Held--Karp DP~\citep{held1962dynamic};
  MTZ formulations via CBC and SCIP
& Attention model~\citep{kool2019attention} \\

\bottomrule
\end{tabular}
\end{adjustbox}

\vspace{5pt}

\begin{adjustbox}{max width=\linewidth,center}
\begin{tabular}{@{}l *{7}{c}@{}}
\toprule
\textbf{Shared backend}
& \textbf{Color.} & \textbf{MaxSAT} & \textbf{MIS} & \textbf{MDS}
& \textbf{Pack.\ LP} & \textbf{MDKP} & \textbf{TSP} \\
\midrule

Gurobi, $10$\,s single-thread~\citep{gurobi2026}
& \checkmark & \checkmark & \checkmark & \checkmark
& \checkmark & \checkmark & \checkmark \\

OR-Tools CP-SAT; GLOP for LP~\citep{perron2025ortools,perron2023cpsat}
& \checkmark & & \checkmark & \checkmark
& \checkmark & \checkmark & \checkmark \\

SCIP MIP; LP~\citep{bestuzheva2023scip}
& \checkmark & & \checkmark & \checkmark
& \checkmark & \checkmark & \checkmark \\

HiGHS MIP; simplex and interior point~\citep{highs2024,huangfu2018parallelizing}
& \checkmark & & \checkmark & \checkmark
& \checkmark & \checkmark & \\

CBC MIP~\citep{forrest2005cbc}
& & & & \checkmark
& & \checkmark & \checkmark \\

CLP LP~\citep{forrest2022clp}
& & & & & \checkmark & & \\

PySAT: SAT-based; RC2 with CaDiCaL 1.9.5, Glucose4, MiniSat 2.2,
and default backend~\citep{ignatiev2019rc2}
& \checkmark & \checkmark & & & & & \\

\bottomrule
\end{tabular}
\end{adjustbox}

\end{table}

\subsection{Compute and API Usage}
\label{app:compute}

Our primary solvers were synthesized with GPT-5.2 Thinking through a remote API, using high reasoning effort and structured outputs; temperature, top-\(p\), and output-length settings were left at their defaults. Across the \(21\) target distributions, synthesis used \(463\) successful API calls and \(16.1\)M tokens, averaging \(22.0\) calls and \(767\)K tokens per target. Of these, \(5.95\)M were prompt tokens and \(10.16\)M completion tokens, including \(8.05\)M reasoning tokens.

We also run the same pipeline with Gemma~4 (\texttt{gemma-4-31b-it}) on a single NVIDIA H100 NVL, both for the full synthesis pipeline (Table~\ref{tab:gemma_vs_gpt}) and for the best-of-\(5\) baseline. The latter used \(3.4\)M tokens across the \(21\) targets and about \(1.9\) GPU-hours. The main benchmark run took approximately \(4.12\) hours.

All deployed solvers and classical baselines run locally on CPU, with no GPU use, training, or fine-tuning. Experiments used a dual-socket AMD EPYC 9554 machine with \(128\) physical cores, \(256\) logical CPUs, and approximately \(1.1\) TiB RAM. Summed across targets, recorded synthesis-stage time was \(73.4\) hours, classical baseline evaluation \(417.3\) hours, and other local synthesis work approximately \(0.6\) hours. These are aggregate per-target stage times rather than elapsed wall-clock time.

GPU use is limited to Gemma generation and the learned ML baselines, both on the same H100 NVL. The ML models are trained separately for each target, using approximately \(6.0\) GPU-hours in total, and timed at inference on both CPU and GPU. As with GPT-5.2 synthesis, these offline costs are excluded from deployment-runtime comparisons.

\begin{table}[t]
\centering
\caption{
\textbf{Compute summary.}
All costs below are incurred once, before deployment, and are excluded from the
runtime ratios reported in the main text --- as are the training costs of the
learned ML baselines. Timing values are aggregate recorded per-target stage
times, not calendar elapsed time.
}
\label{tab:compute_summary}
\small
\setlength{\tabcolsep}{6pt}
\renewcommand{\arraystretch}{1.15}

\newcommand{\sub}[1]{\newline{\scriptsize\textcolor{black!55}{#1}}}

\begin{adjustbox}{max width=\linewidth,center}
\begin{tabular}{@{}
    >{\raggedright\arraybackslash}p{0.34\linewidth}
    l
    >{\raggedright\arraybackslash}p{0.28\linewidth}
    r@{}}
\toprule
\textbf{Offline stage} & \textbf{Runs on} & \textbf{Generation cost} & \textbf{Time} \\
\midrule

GPT-5.2 synthesis
\sub{$21$ targets; $22.0$ calls and $767$K tokens per target}
& remote API + local CPU
& $463$ calls, $16.11$M tokens
  \sub{$5.95$M prompt / $10.16$M completion, of which $8.05$M reasoning}
& $73.4$\,h recorded stage time \\
\addlinespace[4pt]

Gemma~4 generation, best-of-$5$
& $1\times$ H100
& $3.39$M tokens
& ${\approx}1.9$\,GPU-h \\
\addlinespace[4pt]

ML baseline training, all targets
& $1\times$ H100
& ---
& ${\approx}6.0$\,GPU-h \\
\addlinespace[4pt]

Classical baseline pre-synthesis evaluation
& CPU
& ---
& $417.3$\,h \\
\addlinespace[4pt]

Other local work inside synthesis
\sub{analysis execution, solver construction, train/val evaluation}
& CPU
& ---
& ${\approx}0.6$\,h \\

\midrule
\textbf{Deployment}
& CPU
& \textbf{no LLM calls, no training}
& \textbf{ms / instance} \\
\bottomrule
\end{tabular}
\end{adjustbox}

\vspace{3pt}
{\scriptsize
Hardware: $2\times$ AMD EPYC 9554 ($128$ physical / $256$ logical CPUs),
${\approx}1.1$\,TiB RAM, $1\times$ NVIDIA H100 NVL; $83$\,GB stored benchmark
artifacts. Experiments were parallelized across target distributions.
}
\end{table}

\section{Additional Empirical Results}
\label{app:additional_results}

\subsection{ANYCSP and GFlowNet Baselines}
\label{app:anycsp_gflownet}

To strengthen the learned comparison on the problem classes where these methods
naturally apply, we implement compact reproductions of ANYCSP~\citep{tonshoff2023anycsp} for Coloring and
MaxSAT and a GFlowNet~\citep{zhang2023flows} combinatorial-optimization model for MIS and MDS. Each
model is trained separately for each target using only the public training and
validation splits and is evaluated on the same \(500\) held-out instances as our
synthesized solver. These are reduced-budget reproductions in the style of the
corresponding model families rather than the authors' original implementations;
by construction they cover four of the seven problem classes.

\begin{table}[t]
\centering
\caption{
\textbf{Matched learned baselines on full-size held-out targets.}
ANYCSP is evaluated on Coloring and MaxSAT, and GFlowNet on MIS and MDS.
\(Q\) is normalized quality; runtime is per instance in milliseconds.
Speedup is \(T_{\rm baseline}/T_{\rm ours}\).
}
\label{tab:anycsp_gflownet}
\scriptsize
\setlength{\tabcolsep}{3.1pt}
\renewcommand{\arraystretch}{1.08}

\begin{adjustbox}{max width=\linewidth,center}
\begin{tabular}{@{}llcrrrrr@{}}
\toprule
\textbf{Problem} & \textbf{Target}
& \textbf{Baseline}
& \(Q_b\) & \(Q_{\rm ours}\)
& \(T_b\) & \(T_{\rm ours}\)
& \textbf{Speedup} \\
\midrule

Coloring & Ring-template       & ANYCSP   & \(0.7770\) & \(1.0000\) & \(42.5\)   & \(0.2\)   & \(212.5\times\) \\
Coloring & Overlapping-palette & ANYCSP   & \(0.7980\) & \(0.8040\) & \(60.2\)   & \(24.0\)  & \(2.5\times\) \\
Coloring & Separator-trap      & ANYCSP   & \(0.7740\) & \(0.8000\) & \(45.1\)   & \(3.9\)   & \(11.6\times\) \\
MaxSAT   & Community-parity    & ANYCSP   & \(0.9985\) & \(0.9999\) & \(4922\)   & \(24.0\)  & \(205.1\times\) \\
MaxSAT   & Last-clause signal  & ANYCSP   & \(1.0000\) & \(1.0000\) & \(834\)    & \(5.7\)   & \(146.3\times\) \\
MaxSAT   & Latent-backdoor     & ANYCSP   & \(0.9980\) & \(0.9991\) & \(2496\)   & \(40.0\)  & \(62.4\times\) \\
MIS      & Clique-path         & GFlowNet & \(0.8660\) & \(0.9914\) & \(148\)    & \(15.0\)  & \(9.9\times\) \\
MIS      & Core-fringe         & GFlowNet & \(1.0000\) & \(1.0000\) & \(2942\)   & \(7.2\)   & \(408.6\times\) \\
MIS      & Motif-bridge        & GFlowNet & \(0.9130\) & \(0.9860\) & \(397\)    & \(57.0\)  & \(7.0\times\) \\
MDS      & Gateway-hub         & GFlowNet & \(0.9250\) & \(0.9197\) & \(9650\)   & \(5.5\)   & \(1754.5\times\) \\
MDS      & Geometric-anchor    & GFlowNet & \(0.8410\) & \(1.0000\) & \(6894\)   & \(128.0\) & \(53.9\times\) \\
MDS      & Star-kernel         & GFlowNet & \(1.0000\) & \(1.0000\) & \(9951\)   & \(3.5\)   & \(2843.1\times\) \\

\bottomrule
\end{tabular}
\end{adjustbox}
\end{table}

Across the \(12\) matched targets, our solver matches or exceeds baseline quality on \(11/12\) and is faster on all \(12\), with speedups ranging from \(2.5\times\) to over \(2{,}800\times\). The single quality loss is Gateway-hub MDS, where GFlowNet obtains \(0.9250\) versus \(0.9197\) for ours. These comparisons should be read as compact model-family reproductions rather than claims against the authors' fully tuned implementations.

\subsection{Generator Model Ablation}
\label{app:generator_ablation}

We test how much the synthesis results depend on the model used to generate hypotheses, analysis code, and solvers. We rerun the same pipeline with Gemma~4 (\texttt{gemma-4-31b-it}) in place of GPT-5.2, keeping the synthesis procedure and evaluation protocol fixed.

\begin{table}[t]
\centering
\caption{
\textbf{Generator-model ablation.}
The same synthesis pipeline run with GPT-5.2 and with Gemma~4.
\(\Delta Q=Q_{\rm GPT}-Q_{\rm Gem}\), so positive values favor GPT-5.2;
\(S=T_{\rm Gem}/T_{\rm GPT}\), so \(S>1\) means the GPT-generated solver is
faster. The final row averages over all \(21\) targets.
}
\label{tab:gemma_vs_gpt}

\small
\setlength{\tabcolsep}{6pt}
\renewcommand{\arraystretch}{1.12}

\begin{adjustbox}{max width=\linewidth,center}
\begin{tabular}{@{}l rrr rrr@{}}
\toprule
& \multicolumn{3}{c}{\textbf{Quality} \(Q\)}
& \multicolumn{3}{c}{\textbf{Runtime} \(T\) (ms)} \\
\cmidrule(lr){2-4}\cmidrule(l){5-7}
\textbf{Family}
& GPT-5.2 & Gemma~4 & \(\Delta Q\)
& GPT-5.2 & Gemma~4 & \(S\) \\
\midrule

Coloring    & \(0.868\) & \(0.798\) & \(+0.070\) & \(2.7\)  & \(15.5\) & \(5.73\times\) \\
MAXSAT      & \(1.000\) & \(0.910\) & \(+0.090\) & \(17.1\) & \(38.2\) & \(2.23\times\) \\
MIS         & \(0.992\) & \(0.982\) & \(+0.010\) & \(18.8\) & \(25.8\) & \(1.37\times\) \\
MDS         & \(0.973\) & \(0.938\) & \(+0.035\) & \(13.3\) & \(10.2\) & \(0.77\times\) \\
Packing LP  & \(0.994\) & \(0.890\) & \(+0.104\) & \(3.3\)  & \(12.5\) & \(3.80\times\) \\
MDKP        & \(0.973\) & \(0.740\) & \(+0.233\) & \(94.0\) & \(9.3\)  & \(0.099\times\) \\
TSP         & \(0.993\) & \(0.926\) & \(+0.067\) & \(14.6\) & \(28.0\) & \(1.92\times\) \\

\midrule
\textbf{All (21)}
& \(\mathbf{0.971}\) & \(\mathbf{0.883}\) & \(\mathbf{+0.088}\)
& \(\mathbf{12.7}\)  & \(\mathbf{17.5}\)  & \(\mathbf{1.38\times}\) \\
\bottomrule
\end{tabular}
\end{adjustbox}
\end{table}

GPT-5.2 produces higher-quality solvers in every problem family, with mean normalized quality \(0.971\) versus \(0.883\) for Gemma~4. Deployment runtime differs less: GPT-generated solvers are \(1.38\times\) faster overall, but Gemma produces faster solvers for MDS and, especially, MDKP. Thus the generator affects both the specialization found and the resulting quality--runtime tradeoff; the advantage of the stronger model is primarily in solution quality rather than uniformly lower runtime.

\subsection{Matched Hint Ablation}
\label{app:no_hint_ablation}

We next isolate the contribution of explicit hint recovery from iterative code refinement. The matched baseline uses the same GPT-5.2 model, beam search, candidate budget, validation feedback, refinement procedure, repair, and test splits, but omits the hypothesis/hint stage. It still uses a generic analysis program, so the comparison tests explicit structural factorization rather than analysis itself.

\begin{table}[t]
\centering
\caption{
\textbf{Ablation of explicit hint recovery.}
One target per problem family. The no-hint variant removes the explicit
hypothesis stage but retains a generic analysis program; model, search loop,
candidate budget, validation feedback, refinement, repair, and test data are
otherwise matched. \(\Delta Q=Q_{\rm full}-Q_{\rm no\text{-}hint}\); runtime is normalized by the full pipeline.}
\label{tab:no_hint_ablation}
\small
\setlength{\tabcolsep}{7pt}
\renewcommand{\arraystretch}{1.12}

\begin{adjustbox}{max width=\linewidth,center}
\begin{tabular}{@{}ll rr @{\hspace{10pt}} r@{}}
\toprule
\textbf{Problem} & \textbf{Target distribution}
& \textbf{Full} & \textbf{No hint}
& \(\boldsymbol{\Delta Q}\) \\
\midrule

Coloring    & Ring-template       & $1.0000$ & $0.7973$ & $+0.2027$ \\
MDKP        & Latent-class        & $0.9684$ & $0.9026$ & $+0.0658$ \\
TSP         & Clustered Euclidean & $0.9872$ & $0.9806$ & $+0.0066$ \\
MIS         & Clique-path         & $0.9914$ & $0.9865$ & $+0.0049$ \\
MDS         & Geometric-anchor    & $1.0000$ & $0.9974$ & $+0.0026$ \\
MAXSAT      & Latent-backdoor     & $0.9991$ & $0.9991$ & $\phantom{+}0.0000$ \\
Packing LP  & Block-coupled       & $1.0000$ & $1.0000$ & $\phantom{+}0.0000$ \\

\midrule
\multicolumn{2}{@{}l}{\textbf{Mean quality}}
& $\mathbf{0.9923}$ & $\mathbf{0.9519}$ & $\mathbf{+0.0404}$ \\
\multicolumn{2}{@{}l}{{\textbf{Runtime / full (geom. mean)}}}
& $1.00\times$ & $\mathbf{0.52\times}$ & --- \\
\bottomrule
\end{tabular}
\end{adjustbox}
\end{table}

Table~\ref{tab:no_hint_ablation} shows a quality--runtime tradeoff. Explicit hint recovery improves mean quality from \(0.9519\) to \(0.9923\), with the largest gains on Ring-template Coloring and Latent-class MDKP, but the no-hint solvers are about \(1/0.52\approx1.9\times\) faster in geometric-mean deployment runtime. On MAXSAT and Packing LP, both pipelines already reach essentially optimal quality. The hint stage therefore buys substantial quality on the targets where explicit structure recovery matters, rather than providing a uniform runtime improvement.

\subsection{Sample-Conditioned Catalog Selection}
\label{app:catalog_selection}

A natural alternative to synthesizing a new solver is to use the public sample
to choose among existing ones. For each target, we therefore select from the
problem-specific entries in our \(78\)-entry cross-family catalog using
validation performance. One selector maximizes validation quality; a second
minimizes validation runtime among solvers whose quality matches or exceeds that
of our synthesized solver.

\begin{table}[t]
\centering
\caption{
\textbf{Sample-conditioned selection from the existing solver catalog.}
Selectors use validation data and are evaluated on held-out test instances.
Runtime is reported relative to our synthesized solver. Coverage is the number
of targets for which the selection criterion can be satisfied; for incomplete
coverage, mean quality and runtime are computed over covered targets only.
}
\label{tab:catalog_selection}

\small
\setlength{\tabcolsep}{6pt}
\renewcommand{\arraystretch}{1.12}

\begin{adjustbox}{max width=\linewidth,center}
\begin{tabular}{@{}lccc@{}}
\toprule
\textbf{Method}
& \textbf{Mean quality}
& \textbf{Runtime / ours}
& \textbf{Coverage} \\
\midrule

Ours
& $0.9705$
& $1.0\times$
& $21/21$ \\

Maximum-quality selector
& $0.9805$
& $42.6\times$
& $21/21$ \\

Fastest selector at our quality
& $0.9911$
& $27.9\times$
& $19/21$ \\

Per-instance oracle virtual best
& $0.9999$
& $21.4\times$
& $21/21$ \\

\bottomrule
\end{tabular}
\end{adjustbox}
\end{table}

Table~\ref{tab:catalog_selection} shows that selecting from the existing solver catalog does not match the quality--runtime tradeoff of synthesis. The maximum-quality selector raises mean quality from \(0.9705\) to \(0.9805\), but is \(42.6\times\) slower. On the \(19\) covered targets, the fastest selector whose validation quality matches ours reaches mean test quality \(0.9911\) but remains \(27.9\times\) slower; no qualifying solver exists on the other two targets. Even the per-instance oracle virtual best reaches \(0.9999\) quality at \(21.4\times\) our runtime. Thus the gain is not explained simply by using samples to choose among existing algorithms.

The same distinction is sharper on PACE 2025 Dominating Set. Table~\ref{tab:pace_catalog_selection} compares our synthesized solver with the best valid choices from the existing catalog on the 100 private instances.

\begin{table}[t]
\centering
\caption{\textbf{PACE 2025 Dominating Set: synthesis versus catalog selection.}
Lower set size and runtime are better. Exact catalog methods and Gurobi return
no valid solution within the evaluation budget.}
\label{tab:pace_catalog_selection}
\small
\setlength{\tabcolsep}{7pt}
\renewcommand{\arraystretch}{1.12}
\begin{adjustbox}{max width=\linewidth,center}
\begin{tabular}{@{}lcrr@{}}
\toprule
\textbf{Method} & \textbf{Valid} & \textbf{Avg. set size} & \textbf{Runtime} \\
\midrule
Ours & $100/100$ & $231{,}595$ & $2.89$ s \\
Best catalog heuristic & $100/100$ & $234{,}674$ & $8.3$ s \\
Oracle best-valid catalog pick & $100/100$ & $234{,}673$ & $8.3$ s \\
Catalog exact / Gurobi & $0/100$ & --- & timeout \\
\bottomrule
\end{tabular}
\end{adjustbox}
\end{table}

On PACE, the best valid catalog heuristic is both slower and about \(1.3\%\) worse in set size, while exact catalog methods and Gurobi return no valid solution within the budget. This provides an external-scale example where selecting an existing solver does not reproduce the synthesized solver's operating point.

\subsection{Full Per-Target Results}
\label{app:full_results}

\begin{table}[t]
\centering
\caption{
\textbf{Held-out results on all \(21\) target distributions.}
\(Q\) is normalized quality, Opt.\ is optimality rate, and \(T\) is mean
per-instance runtime in milliseconds.
}
\label{tab:full_results}

\scriptsize
\setlength{\tabcolsep}{2.6pt}
\renewcommand{\arraystretch}{1.08}
\providecolor{ourpurple}{HTML}{ECE2F7}

\newcommand{\prow}[1]{%
  \addlinespace[3pt]%
  \multicolumn{16}{@{}l}{\itshape\textcolor{black!65}{#1}}\\
  \addlinespace[1pt]}
\newcommand{\dst}[1]{\hspace{0.6em}#1}

\begin{adjustbox}{max width=\linewidth,center}
\begin{tabular}{@{}l
@{\hspace{6pt}} ccc
@{\hspace{7pt}} ccc
@{\hspace{7pt}} ccc
@{\hspace{7pt}} ccc
@{\hspace{5pt}}
>{\columncolor{ourpurple}}c
>{\columncolor{ourpurple}}c
>{\columncolor{ourpurple}}c@{}}

\toprule
&
\multicolumn{3}{c}{\textbf{Best heuristic}} &
\multicolumn{3}{c}{\textbf{ML}} &
\multicolumn{3}{c}{\textbf{Gurobi (10\,s)}} &
\multicolumn{3}{c}{\textbf{Exact (10\,s)}} &
\multicolumn{3}{c}{\bfseries Ours} \\
\cmidrule(lr){2-4}
\cmidrule(lr){5-7}
\cmidrule(lr){8-10}
\cmidrule(lr){11-13}
\cmidrule(l){14-16}

\textbf{Target distribution}
& \(Q\) & Opt. & \(T\)
& \(Q\) & Opt. & \(T\)
& \(Q\) & Opt. & \(T\)
& \(Q\) & Opt. & \(T\)
& \(Q\) & Opt. & \(T\) \\
\midrule

\prow{Coloring}

\dst{Ring-template}
& $.746$ & $.002$ & $3699$
& $.992$ & $.962$ & $2.2$
& $.976$ & $.882$ & $4136$
& $1.00$ & $1.00$ & $31$
& $1.00$ & $1.00$ & $0.2$ \\

\dst{Overlapping-palette}
& $.747$ & $.004$ & $3476$
& $.999$ & $.997$ & $1.9$
& $.878$ & $.392$ & $8366$
& $1.00$ & $1.00$ & $80$
& $.804$ & $.022$ & $24$ \\

\dst{Separator-trap}
& $.749$ & $.000$ & $3616$
& $.995$ & $.974$ & $2.1$
& $.933$ & $.666$ & $6877$
& $1.00$ & $1.00$ & $99$
& $.800$ & $.000$ & $4$ \\

\prow{MAXSAT}

\dst{Community-parity}
& $.914$ & $.000$ & $3733$
& $1.00$ & $.986$ & $21$
& $1.00$ & $.992$ & $8712$
& $1.00$ & $1.00$ & $26$
& $1.00$ & $.876$ & $24$ \\

\dst{Last-clause signal}
& $.932$ & $.000$ & $3645$
& $1.00$ & $1.00$ & $14$
& $1.00$ & $1.00$ & $8325$
& $1.00$ & $1.00$ & $8$
& $1.00$ & $1.00$ & $6$ \\

\dst{Latent-backdoor}
& $.931$ & $.000$ & $3791$
& $1.00$ & $.999$ & $20$
& $1.00$ & $.996$ & $2461$
& $1.00$ & $1.00$ & $27$
& $1.00$ & $.896$ & $38$ \\

\prow{MIS}

\dst{Clique-path}
& $.876$ & $.040$ & $10000$
& $.996$ & $.965$ & $1.6$
& $1.00$ & $1.00$ & $2405$
& $1.00$ & $1.00$ & $1355$
& $.991$ & $.594$ & $15$ \\

\dst{Core-fringe trap}
& $.908$ & $.908$ & $10000$
& $1.00$ & $1.00$ & $1.6$
& $1.00$ & $1.00$ & $1555$
& $1.00$ & $1.00$ & $145$
& $1.00$ & $1.00$ & $8$ \\

\dst{Motif-bridge}
& $.874$ & $.044$ & $10000$
& $.988$ & $.878$ & $1.6$
& $.998$ & $.890$ & $6757$
& $.994$ & $.994$ & $2136$
& $.986$ & $.384$ & $57$ \\

\prow{MDS}

\dst{Gateway-hub}
& $.839$ & $.000$ & $10000$
& $.935$ & $.702$ & $6.6$
& $1.00$ & $1.00$ & $8099$
& $1.00$ & $1.00$ & $113$
& $.920$ & $.000$ & $5$ \\

\dst{Geometric-anchor}
& $.784$ & $.298$ & $8706$
& $.945$ & $.790$ & $4.5$
& $1.00$ & $1.00$ & $3083$
& $1.00$ & $1.00$ & $2093$
& $1.00$ & $1.00$ & $130$ \\

\dst{Star-kernel}
& $.930$ & $.930$ & $10000$
& $.977$ & $.959$ & $5.8$
& $1.00$ & $1.00$ & $8162$
& $1.00$ & $1.00$ & $52$
& $1.00$ & $1.00$ & $3$ \\

\prow{Packing LP}

\dst{Block-coupled}
& $.792$ & $.000$ & $10000$
& $.796$ & $.000$ & $1.0$
& $1.00$ & $1.00$ & $8628$
& $1.00$ & $1.00$ & $125$
& $1.00$ & $1.00$ & $1.5$ \\

\dst{Active-resource}
& $.610$ & $.000$ & $10000$
& $.724$ & $.000$ & $1.0$
& $1.00$ & $1.00$ & $12676$
& $1.00$ & $1.00$ & $124$
& $.981$ & $.000$ & $17$ \\

\dst{Single-bottleneck}
& $.803$ & $.000$ & $10000$
& $.903$ & $.000$ & $1.0$
& $1.00$ & $1.00$ & $7636$
& $1.00$ & $1.00$ & $122$
& $1.00$ & $1.00$ & $1.4$ \\

\prow{MDKP}

\dst{Decoy-complement}
& $.967$ & $.016$ & $5330$
& $.982$ & $.319$ & $6.5$
& $1.00$ & $1.00$ & $2499$
& $1.00$ & $1.00$ & $1084$
& $.951$ & $.000$ & $41$ \\

\dst{Latent-class}
& $.966$ & $.010$ & $3660$
& $.981$ & $.291$ & $6.5$
& $1.00$ & $.746$ & $4955$
& $1.00$ & $.544$ & $6306$
& $.968$ & $.000$ & $116$ \\

\dst{Single-resource}
& $.960$ & $.666$ & $4263$
& $.958$ & $.019$ & $4.2$
& $1.00$ & $1.00$ & $2593$
& $1.00$ & $1.00$ & $233$
& $1.00$ & $1.00$ & $176$ \\

\prow{TSP}

\dst{Clustered Euclidean}
& $1.00$ & $.910$ & $1544$
& $.782$ & $.000$ & $888$
& $1.00$ & $.930$ & $1553$
& $1.00$ & $.662$ & $136$
& $.987$ & $.002$ & $67$ \\

\dst{Latent-metric}
& $1.00$ & $.848$ & $1490$
& $.970$ & $.282$ & $891$
& $1.00$ & $.966$ & $1566$
& $1.00$ & $.758$ & $118$
& $.992$ & $.306$ & $235$ \\

\dst{Paired-ribbon zigzag}
& $1.00$ & $1.00$ & $51$
& $.976$ & $.700$ & $888$
& $1.00$ & $1.00$ & $2071$
& $.499$ & $.000$ & $10993$
& $1.00$ & $1.00$ & $0.2$ \\

\bottomrule
\end{tabular}
\end{adjustbox}


\end{table}

Table~\ref{tab:full_results} reports the full held-out test results behind the
headline summary in Table~\ref{tab:headline_results}, giving per-target
normalized quality, optimality rate, and mean runtime for every baseline group.
Two points are worth noting when reading it against the headline table. First,
the heuristic column here is the \emph{quality-best} heuristic per target,
whereas the headline runtime ratio $R_{\rm Heur}$ uses the \emph{fast
high-quality} heuristic (the fastest within $95\%$ of best quality), so the two
heuristic runtimes are not directly comparable. Second, the ML baseline is fast
and high-quality on several graph families but degrades on Packing LP and TSP,
consistent with its aggregate position in Table~\ref{tab:headline_results}. The
aggregate quality lifts and runtime ratios in the headline table are computed
directly from these $21$ target-level values.

\subsection{PACE 2025 Dominating Set Comparison}
\label{app:pace_ds}

We evaluate on the PACE 2025 Dominating Set heuristic track, where the goal is to return a small dominating set of an undirected graph. PACE released \(100\) public instances and, after the competition, \(100\) private instances from a similar distribution. We split the public set \(50/50\) for training and validation and evaluate only on the released private set. These results are local comparisons with released PACE instances and solvers, not official competition scores.

The private graphs are large and sparse, ranging from tens of thousands to more than four million vertices, with median size about \(6.5\times10^5\) vertices and \(9.4\times10^5\) edges. Table~\ref{tab:pace_graph_stats} gives the full size distribution. The learned ML baselines are not included here because the largest graphs exceed the scale handled by those models.

\begin{table}[t]
\centering
\caption{
\textbf{PACE private-test graph statistics,} over the $100$ released private
Dominating Set heuristic instances. Quantiles are computed separately for each
row, so entries in a column need not come from the same instance; in
particular, the average-degree row is not the ratio of the two rows above it.
}
\label{tab:pace_graph_stats}
\small
\setlength{\tabcolsep}{8pt}
\renewcommand{\arraystretch}{1.15}
\begin{adjustbox}{max width=\linewidth,center}
\begin{tabular}{@{}l rrrrr@{}}
\toprule
\textbf{Quantity} & \textbf{Min.} & \textbf{25th pct.} & \textbf{Median}
& \textbf{75th pct.} & \textbf{Max.} \\
\midrule
Vertices \(|V|\)
& $18{,}046$ & $21{,}400$ & $653{,}867$ & $1{,}261{,}963$ & $4{,}221{,}675$ \\
Edges \(|E|\)
& $54{,}138$ & $382{,}083$ & $939{,}263$ & $1{,}831{,}990$ & $10{,}356{,}300$ \\
Average degree \(2|E|/|V|\)
& $2.28$ & $2.57$ & $3.74$ & $13.00$ & $48.00$ \\
\bottomrule
\end{tabular}
\end{adjustbox}

\end{table}

For each private instance, we run our synthesized solver and compare it with the released PACE heuristic solvers: Fontan--Verger, Root, Shadoks, AEG Heidelberg, Greeduce, and Swats. We verify every returned solution and report locally measured runtime. Because Dominating Set is a minimization problem, smaller returned sets indicate better quality.

\begin{table}[t]
\centering
\caption{
\textbf{PACE 2025 Dominating Set results.}
Lower size and runtime are better. Our solver is valid on all \(100\) private
instances and is roughly \(75\)--\(125\times\) faster than the released PACE
solvers, whose dominating sets are about \(3\%\) smaller. Size ratio is
our size divided by the comparator size. Swats statistics are computed over
the \(75\) instances on which it returns a valid solution.
}
\label{tab:pace_runtime_quality}

\small
\setlength{\tabcolsep}{5pt}
\renewcommand{\arraystretch}{1.12}

\begin{adjustbox}{max width=\linewidth}
\begin{tabular}{@{}l c r r c r r c@{}}
\toprule
\textbf{Solver}
& \textbf{Valid}
& \textbf{Total size} \(\downarrow\)
& \textbf{Avg.\ size} \(\downarrow\)
& \textbf{Size ratio}
& \textbf{Time (s)} \(\downarrow\)
& \textbf{Speedup}
& \textbf{Wins vs.\ ours} \\
\midrule

\rowcolor{ourpurple}
Ours
& $100/100$
& $23.16$M
& $231{,}595$
& $1.000$
& $2.89$
& $1.0\times$
& -- \\

\midrule

AEG Heidelberg~\citep{aegheidelberg2025pace}
& $100/100$
& $22.41$M
& $224{,}086$
& $1.034$
& $350.14$
& $121.0\times$
& $99/100$ \\

Fontan--Verger~\citep{fontan_et_al:LIPIcs.IPEC.2025.36}
& $100/100$
& $22.41$M
& $224{,}107$
& $1.033$
& $286.24$
& $98.9\times$
& $100/100$ \\

Root~\citep{root_pace2025_solver}
& $100/100$
& $22.41$M
& $224{,}108$
& $1.033$
& $360.42$
& $124.5\times$
& $100/100$ \\

Shadoks~\citep{dafonseca_et_al:LIPIcs.IPEC.2025.34}
& $100/100$
& $22.43$M
& $224{,}306$
& $1.032$
& $316.07$
& $109.2\times$
& $100/100$ \\

Greeduce~\citep{polak2025greeduce}
& $100/100$
& $22.47$M
& $224{,}699$
& $1.031$
& $300.86$
& $104.0\times$
& $91/100$ \\

Swats~\citep{swat:LIPIcs.IPEC.2025.38}
& $75/100$
& $15.77$M
& $210{,}237$
& $1.028$
& $218.11$
& $75.4\times$
& $75/75$ \\

\bottomrule
\end{tabular}
\end{adjustbox}
\end{table}

The results show a clear speed--quality tradeoff. The released PACE solvers generally return smaller dominating sets, with our solutions about \(3.3\%\) larger than the strongest fully valid submissions. In return, our solver is roughly \(100\times\) faster and returns valid solutions on all \(100\) instances. Greeduce outperforms it on \(91/100\) instances, while Swats is valid on only \(75/100\). Thus, on this benchmark, synthesis produces a fully valid distribution-specific heuristic that gives up a few percent in solution quality for roughly two orders of magnitude in runtime.

{\bf Matched-budget anytime comparison.\enspace}
Because the released PACE solvers are anytime methods, we also run Root,
Greeduce, and Fontan--Verger under nominal \(3\), \(10\), and \(30\) second budgets. Table~\ref{tab:pace_anytime} reports the mean per-instance normalized quality score over all \(100\) private instances, assigning zero quality to invalid outputs; our solver is shown at its natural \(2.90\)s operating point.

\begin{table}[t]
\centering
\caption{\textbf{PACE 2025 Dominating Set under matched solver budgets.}
Entries report the mean per-instance normalized quality score over all \(100\)
private instances; larger is better and invalid outputs receive zero. Our solver
has \(Q=0.2285\) at \(2.90\)s.}
\label{tab:pace_anytime}
\small
\setlength{\tabcolsep}{8pt}
\renewcommand{\arraystretch}{1.12}
\begin{adjustbox}{max width=\linewidth,center}
\begin{tabular}{@{}lccc@{}}
\toprule
\textbf{Solver} & \textbf{3 s} & \textbf{10 s} & \textbf{30 s} \\
\midrule
Root & $0.1898$ & $0.2244$ & $0.2373$ \\
Greeduce & $0.0805$ & $0.2134$ & $0.2326$ \\
Fontan--Verger & $0.1847$ & $0.2133$ & $0.2348$ \\
\midrule
Ours & \multicolumn{3}{c}{$0.2285$ at $2.90$ s} \\
\bottomrule
\end{tabular}
\end{adjustbox}
\end{table}

Table~\ref{tab:pace_anytime} shows that none of the three released solvers matches our all-instance quality at either the \(3\)s or \(10\)s budget. At \(30\)s, all three surpass it. The synthesized solver therefore provides a strong low-latency operating point, while the competition solvers benefit from
additional search time.

\subsection{Matched-Budget MIS Comparison}
\label{app:kamis_anytime}

We also compare against the OnlineMIS and ReduMIS heuristic modes of KaMIS on the three MIS targets. Table~\ref{tab:kamis_anytime} reports representative operating points using full wall-clock time.

\begin{table}[t]
\centering
\caption{\textbf{Matched-budget comparison with KaMIS heuristic modes.}
Each entry reports normalized quality and wall-clock runtime.}
\label{tab:kamis_anytime}
\small
\setlength{\tabcolsep}{7pt}
\renewcommand{\arraystretch}{1.12}
\begin{adjustbox}{max width=\linewidth,center}
\begin{tabular}{@{}lccc@{}}
\toprule
\textbf{Target} & \textbf{Ours} & \textbf{OnlineMIS} & \textbf{ReduMIS} \\
\midrule
Clique-path & $0.9914$ @ $15.3$ ms & $0.9952$ @ $15.1$ ms & $0.9999$ @ $551$ ms \\
Core-fringe & $1.0000$ @ $7.2$ ms & $1.0000$ @ $37$ ms & $1.0000$ @ $37$ ms \\
Motif-bridge & $0.9860$ @ $57.3$ ms & $0.9941$ @ $13.9$ ms & $1.0000$ @ $1056$ ms \\
\bottomrule
\end{tabular}
\end{adjustbox}
\end{table}

The comparison is mixed. On Core-fringe, our solver reaches optimal quality in \(7.2\) ms, about \(5\times\) faster than either KaMIS mode. OnlineMIS is slightly stronger at the same latency on Clique-path and is both faster and higher-quality on Motif-bridge; ReduMIS reaches near-optimal or optimal quality on all three targets but uses substantially more time on Clique-path and Motif-bridge.

\subsection{Distributional Algorithmic Complexity}
\label{app:distributional_complexity}

\begin{table*}[t]
\centering
\caption{
\textbf{Computation used by the synthesized solvers across the \(21\) target distributions.}
For each problem class, we show the generic exact-search complexity and the
distribution-specific computation used by the synthesized solver. A slash
denotes a learned fast path followed by its fallback. Implementation budgets such as \(B,R,P,F,S,\) and \(K\) are fixed independently
of input size. Corresponding exact-solver runtimes are reported in
Table~\ref{tab:full_results}.
}
\label{tab:dac_all_distributions}

\small
\setlength{\tabcolsep}{4pt}
\renewcommand{\arraystretch}{1.22}

\newcommand{\amb}[2]{%
  \addlinespace[3pt]%
  \multicolumn{2}{@{}l}{%
    \textbf{#1}\hspace{0.7em}
    \textcolor{black!55}{\textit{generic:} \(#2\)}%
  }\\
  \addlinespace[1pt]}
\newcommand{\dacdst}[1]{\hspace{0.8em}#1}

\begin{adjustbox}{max width=\linewidth,center}
\begin{tabular}{@{}l @{\hspace{16pt}} l@{}}
\toprule
\textbf{Distribution}
& \textbf{Synthesized computation} \\
\midrule

\amb{Coloring}{O^*(\kappa^n)}
\dacdst{Ring-template}
& $O(m)$ \,/\, $O(n^2+m+B_{\rm rep}\Delta_{\max})$ \\
\dacdst{Overlapping-palette}
& $O\!\left(n+m+R_{\rm ds}(n^2+m)+P_{\rm elim}\kappa(n+m)\right)$ \\
\dacdst{Separator-trap}
& $O\!\left(n+m+R_{\rm ds}(n^2+m)+B_{\rm recolor}\right)$ \\

\amb{MAXSAT}{O^*(2^v)}
\dacdst{Community-parity}
& $O\!\left(|F|\ell+R_{\rm ws}B_{\rm flip}\ell\Delta_{\rm occ}\right)$ \\
\dacdst{Last-clause signal}
& $O\!\left(|F|\ell+R_{\rm ws}B_{\rm flip}\ell\Delta_{\rm occ}
+B_{\rm patch}\ell\Delta_{\rm occ}\right)$ \\
\dacdst{Latent-backdoor}
& $O\!\left(|F|\ell+R_{\rm ws}B_{\rm flip}\ell\Delta_{\rm occ}\right)$ \\

\amb{MIS}{O^*(2^n)}
\dacdst{Clique-path}
& $O\!\left(n+m+P_{\rm ord}n\log n+P_{\rm gr}(n+m)+T_{\rm ARW}\right)$ \\
\dacdst{Core-fringe trap}
& $O\!\left(n+m+\textstyle\sum_j4^{t_j}+kq\right)$ \\
\dacdst{Motif-bridge}
& $O\!\left(n+m+P_{\rm gr}(n+m)+B_{\rm loc}(n+m)
+B_{\rm kick}(n+m)+T_{\rm tiny}\right)$ \\

\amb{MDS}{O^*(2^n)}
\dacdst{Gateway-hub}
& $O\!\left(n+m+T_{\rm heap\text{-}greedy}+T_{\rm prune}\right)$ \\
\dacdst{Geometric-anchor}
& $O(n+m)$ \,/\, $O\!\left((1+B_{\rm geo})(n+m)\right)$ \\
\dacdst{Star-kernel}
& $O\!\left(n+m+\textstyle\sum_jc_j2^{t_j}
+T_{\rm greedy}+T_{\rm prune}\right)$ \\

\amb{Packing LP}{T_{\rm LP}(N,r,L)}
\dacdst{Block-coupled}
& $O(Nr+N\log N)$ \\
\dacdst{Active-resource}
& $O\!\left(Nr+P_{\rm LP}(Nr+N\log N)\right)$ \\
\dacdst{Single-bottleneck}
& $O(Nr+N\log N)$ \\

\amb{MDKP}{O^*(2^N)}
\dacdst{Decoy-complement}
& $O\!\left(Nr+P_{\rm score}(Nr+N\log N)
+T_{\rm add/drop}+T_{\rm repair}\right)$ \\
\dacdst{Latent-class}
& $O\!\left(Nr+N\log N+
P_{\rm price}T_{\rm iter}(Nr+N\log N+T_{\rm repair})\right)$ \\
\dacdst{Single-resource}
& $O\!\left(Nr+N\log N+K_{\rm cand}C_b+T_{\rm repair}\right)$ \\

\amb{TSP}{O(n^2 2^n)}
\dacdst{Clustered Euclidean}
& $O\!\left(n^2\log n+B_{\rm cand}n^2+B_{\rm full}n^2\right)$ \\
\dacdst{Latent-metric}
& $O\!\left(n^2+F(kn^2+n\log n)+S(2^k k+kn)
+P_{\rm imp}B_{\rm imp}n^2\right)$ \\
\dacdst{Paired-ribbon zigzag}
& $O(n\log n)$ \,/\, $O\!\left((1+B_{\rm 2opt})n^2\right)$ \\

\bottomrule
\end{tabular}
\end{adjustbox}

\end{table*}

This appendix expands the computation patterns summarized in
Table~\ref{tab:dac_all_distributions}. The comparison is distributional rather
than worst-case: the synthesized programs do not improve the worst-case
complexity of the underlying problem classes. Instead, they exploit regularities
shared by instances from a fixed distribution to replace generic search with
smaller, distribution-specific computations.

A generic exact method must remain valid on arbitrary instances and therefore
searches the full space of colorings, Boolean assignments, vertex subsets,
item subsets, or tours. Packing LP is different in that the ambient problem is
polynomial-time solvable, but a generic LP solver still handles the full
\(N\)-variable, \(r\)-constraint program. After observing samples, the
synthesized solvers can instead exploit recurring structure through templates,
restricted candidate sets, local search, surrogate scores, or a small collection
of structured candidates. The speedups in Table~\ref{tab:dac_all_distributions}
come from these changes in deployed computation.

{\bf Notation.\enspace}
For graph problems, \(n\) and \(m\) denote the numbers of vertices and edges,
\(\Delta_{\max}\) the maximum degree, and \(\kappa\) the number of colors
considered by the generic Coloring search. For MaxSAT, \(v\) is the number of
variables, \(|F|\) the number of clauses, \(\ell\) the maximum clause width, and
\(\Delta_{\rm occ}\) the maximum variable-occurrence count. Packing LP and MDKP
use \(N\) items, \(r\) resource constraints, and input bit-length \(L\); TSP uses
\(n\) cities.

The remaining symbols denote fixed budgets used by the synthesized programs.
They do not grow with the input size and typically bound restarts, repair steps,
local-search passes, or candidate sets:

\begin{itemize}[leftmargin=1.4em,itemsep=1pt,topsep=2pt]

\item \emph{Coloring:}
\(B_{\rm rep}\) is the repair budget,
\(R_{\rm ds}\) the number of DSATUR-style restarts,
\(P_{\rm elim}\) the number of color-elimination passes, and
\(B_{\rm recolor}\) the recoloring budget.

\item \emph{MaxSAT:}
\(R_{\rm ws}\) is the number of WalkSAT-style restarts,
\(B_{\rm flip}\) the number of flips per restart, and
\(B_{\rm patch}\) the patch-search budget.

\item \emph{MIS:}
\(P_{\rm ord}\) and \(P_{\rm gr}\) count ordering rules and greedy constructions,
\(B_{\rm loc}\) and \(B_{\rm kick}\) bound local search and kick-repair, and
\(T_{\rm ARW}\) and \(T_{\rm tiny}\) denote bounded auxiliary subroutines.
For Core-fringe, \(t_j\) is the size of residual fringe component \(j\),
\(k\) the number of core choices, and \(q\) the number of candidate
configurations scored.

\item \emph{MDS:}
\(T_{\rm heap\text{-}greedy}\), \(T_{\rm greedy}\), and \(T_{\rm prune}\)
denote greedy and pruning subroutines, while \(B_{\rm geo}\) bounds
geometric-anchor completion. For Star-kernel, \(t_j\) is the number of residual
targets in component \(j\), and \(c_j\) the number of candidate covers.

\item \emph{Packing LP:}
\(P_{\rm LP}\) is the number of density or active-resource rules considered.

\item \emph{MDKP:}
\(P_{\rm score}\) is the number of surrogate scores,
\(P_{\rm price}\) the number of resource-price vectors,
\(T_{\rm iter}\) the cost of one price-update and repair iteration,
\(K_{\rm cand}\) the candidate-set size,
\(C_b\) the effective one-dimensional bottleneck capacity, and
\(T_{\rm add/drop}\) and \(T_{\rm repair}\) the bounded improvement and repair
costs.

\item \emph{TSP:}
\(B_{\rm cand}\) and \(B_{\rm full}\) bound candidate and full 2-opt search,
\(F\) is the number of coordinate frames,
\(k\) the number of stripes or groups,
\(S\) the number of selected stripe frames,
\(P_{\rm imp}\) the number of tours improved,
\(B_{\rm imp}\) the number of improvement passes per tour, and
\(B_{\rm 2opt}\) the paired-ribbon fallback budget.

\end{itemize}

{\bf Graph coloring.\enspace}
The generic search space has size \(O^*(\kappa^n)\). The synthesized solvers instead rely on DSATUR-style construction and bounded repair, where
\(T_{\rm DSATUR}(n,m,\kappa)=O(n^2+m+n\kappa)=O(n^2+m)\) for \(\kappa\le n\).
For \emph{Ring-template}, the main path checks a learned coloring pattern against
the edge set in \(O(m)\); if the pattern fails, the solver invokes bounded repair
or the greedy fallback, for \(O(n^2+m+B_{\rm rep}\Delta_{\max})\).
\emph{Overlapping-palette} uses a bounded number of DSATUR attempts followed by
color elimination,
\[
O\!\left(n+m+R_{\rm ds}(n^2+m)+P_{\rm elim}\kappa(n+m)\right),
\]
whereas \emph{Separator-trap} uses capped recoloring, giving
\(O(n+m+R_{\rm ds}(n^2+m)+B_{\rm recolor})\).

{\bf MaxSAT.\enspace}
The ambient search ranges over \(2^v\) assignments. Here the synthesized solvers
use statistics from the training instances to choose an initial assignment and
then make a bounded number of local repairs. Building the occurrence lists and
initial assignment costs \(O(|F|\ell)\), and each flip costs
\(O(\ell\Delta_{\rm occ})\). Thus \emph{Community-parity} and
\emph{Latent-backdoor} require
\(O(|F|\ell+R_{\rm ws}B_{\rm flip}\ell\Delta_{\rm occ})\), while
\emph{Last-clause signal} adds
\(O(B_{\rm patch}\ell\Delta_{\rm occ})\) to repair variables near the remaining
unsatisfied clauses. The useful structure here is therefore not a smaller
worst-case search space, but an initialization for which bounded local search is
usually enough.

{\bf Maximum independent set.\enspace}
Rather than search all \(2^n\) vertex subsets, the synthesized solvers construct
and refine a small collection of candidate independent sets.
\emph{Clique-path} combines learned vertex orderings, greedy maximal sets, and
ARW-style improvement, for
\(O(n+m+P_{\rm ord}n\log n+P_{\rm gr}(n+m)+T_{\rm ARW})\).
\emph{Core-fringe trap} decomposes the fringe into small components and
enumerates only within those components:
\[
O\!\left(n+m+\textstyle\sum_j4^{t_j}+kq\right).
\]
The exponential term therefore depends on the residual component sizes \(t_j\),
not on the full graph size \(n\). \emph{Motif-bridge} instead uses greedy
constructions, bounded local improvement, kick moves, and an optional small
exact step, with
\(O(n+m+P_{\rm gr}(n+m)+B_{\rm loc}(n+m)+B_{\rm kick}(n+m)+T_{\rm tiny})\).

{\bf Minimum dominating set.\enspace}
The generic search considers \(2^n\) vertex subsets.
\emph{Gateway-hub} exploits high-degree hubs, leaves, and isolated vertices,
followed by greedy completion and pruning, for
\(O(n+m+T_{\rm heap\text{-}greedy}+T_{\rm prune})\).
\emph{Geometric-anchor} checks a small anchor pattern in \(O(n+m)\) and uses
bounded greedy completion when needed, giving
\(O((1+B_{\rm geo})(n+m))\).
\emph{Star-kernel} first applies forced leaf-neighbor choices and then solves
only the remaining small components. Its residual work is
\(\sum_j c_j2^{t_j}\), so the exponential part is confined to those components.

{\bf Packing LP.\enspace}
Packing LP is different because the underlying problem is already
polynomial-time solvable. The comparison is between a generic LP solve,
\(T_{\rm LP}(N,r,L)\), and a simpler rule adapted to the observed resource
structure. For \emph{Block-coupled} and \emph{Single-bottleneck}, the solver
scores items against the effective resource and sorts them in
\(O(Nr+N\log N)\). \emph{Active-resource} evaluates \(P_{\rm LP}\) such scoring
rules, for \(O(Nr+P_{\rm LP}(Nr+N\log N))\). In these distributions, a small
set of density rules is sufficient to avoid a full general-purpose LP solve.

{\bf Multidimensional knapsack.\enspace}
The ambient search has size \(2^N\). The synthesized solvers instead use
surrogate scores, restricted candidate sets, bounded dynamic programming, and
repair. \emph{Decoy-complement} scores and greedily packs items before bounded
add/drop improvement, for
\(O(Nr+P_{\rm score}(Nr+N\log N)+T_{\rm add/drop}+T_{\rm repair})\).
\emph{Latent-class} iterates over a small set of resource-price vectors,
\[
O\!\left(
Nr+N\log N+
P_{\rm price}T_{\rm iter}(Nr+N\log N+T_{\rm repair})
\right),
\]
while \emph{Single-resource} identifies a dominant constraint and runs
one-dimensional DP on a bounded candidate set, with
\(O(Nr+N\log N+K_{\rm cand}C_b+T_{\rm repair})\).

{\bf Traveling salesman problem.\enspace}
Held--Karp takes \(O(n^2 2^n)\) time. The synthesized solvers instead exploit
the recurring geometry of each distribution to construct a small set of tours
and improve them locally. \emph{Clustered Euclidean} builds geometric or
nearest-neighbor tours and applies bounded candidate and full 2-opt, for
\(O(n^2\log n+B_{\rm cand}n^2+B_{\rm full}n^2)\).
\emph{Latent-metric} evaluates \(F\) coordinate frames with \(k\) groups,
solves a small stripe-order problem for the selected \(S\) frames, and improves
a bounded number of tours:
\[
O\!\left(
n^2+F(kn^2+n\log n)+S(2^k k+kn)
+P_{\rm imp}B_{\rm imp}n^2
\right).
\]
Since \(k\) is fixed and small, the \(2^k\) term is independent of the asymptotic
growth in \(n\). Finally, \emph{Paired-ribbon zigzag} detects the two-rail
structure, sorts points along the rails, and pairs endpoints in \(O(n\log n)\).
When that structure is not detected, it falls back to nearest-neighbor
construction followed by bounded 2-opt, for
\(O((1+B_{\rm 2opt})n^2)\).

{\bf Summary.\enspace}
The same pattern appears across the seven problem classes: samples reveal
structure that lets the deployed solver work with templates, good
initializations, small residual components, dominant resources, restricted
candidate sets, or geometric orderings instead of carrying out the full generic
computation. These procedures do not improve the worst-case complexity of the
underlying problems, and several retain bounded repair or fallback paths. Their
advantage is specific to the distributions on which they were synthesized.

\subsection{Zero-Sample versus Sample-Conditioned Synthesis}
\label{app:with_without_data}

We isolate the contribution of public samples with a paired ablation. The
zero-sample variant runs the same synthesis pipeline with no public training
instances; the sample-conditioned variant uses the default $m=64$ training
instances from our main experiments (we fix this rather than tuning the sample
size). Each row averages $10$ paired synthesis seeds, so the two variants are
compared under matched randomness.

Table~\ref{tab:with_without_data_all} shows quality is near-saturated under both
variants, so the observable effect of sample access on these distributions is
not feasibility or final quality but \emph{deployment runtime}. The
sample-conditioned solver is at least as fast on $10/12$ targets and strictly
faster on $9/12$, with large gains on Single-resource MDKP ($36\times$),
Paired-ribbon zigzag TSP ($18\times$), and Separator-trap Coloring ($6\times$).
The one substantial regression is Decoy-complement MDKP, where the
sample-conditioned solver is slower ($0.36\times$) despite a marginal quality
gain. We read this as evidence for sample-conditioned compilation: public samples
often help synthesis discover faster deployment code even when zero-sample
synthesis already attains high quality---though the effect is
distribution-dependent rather than a monotone function of sample count.

\begin{table}[t]
\centering
\caption{
\textbf{Zero-sample versus sample-conditioned synthesis.}
The same pipeline run with no training instances (\(m=0\)) and with \(64\)
(\(m=64\)) on \(12\) targets from four problem families, averaged over \(10\)
paired synthesis seeds. Quality is near-identical in the two settings, so we
report \(\Delta Q=10^3(Q_{64}-Q_0)\); speedup \(T_0/T_{64}>1\) favors sample
conditioning.
}
\label{tab:with_without_data_all}

\footnotesize
\setlength{\tabcolsep}{4pt}
\renewcommand{\arraystretch}{1.12}

\newcommand{\zsgrp}[1]{%
  \addlinespace[2pt]%
  \multicolumn{5}{@{}l}{\itshape\textcolor{black!65}{#1}}\\
  \addlinespace[1pt]}
\newcommand{\zsdst}[1]{\hspace{0.6em}#1}
\newcommand{\zshead}{%
\toprule
& \(\Delta Q{\times}10^3\) & \multicolumn{2}{c}{Runtime (ms)} & Speedup \\
\cmidrule(lr){2-2}\cmidrule(lr){3-4}\cmidrule(l){5-5}
\textbf{Target distribution} & & \(m{=}0\) & \(m{=}64\) & \(T_0/T_{64}\) \\
\midrule}

\begin{minipage}[t]{0.49\linewidth}
\centering
\begin{adjustbox}{max width=\linewidth}
\begin{tabular}{@{}l r @{\hspace{8pt}} rr r@{}}
\zshead

\zsgrp{Coloring}
\zsdst{Ring-template}        & $\phantom{+}0.00$ & $0.101$ & $0.072$ & $1.40\times$ \\
\zsdst{Overlapping-palette} & $-0.08$           & $0.076$ & $0.077$ & $0.99\times$ \\
\zsdst{Separator-trap}  & $\phantom{+}0.00$ & $0.603$ & $0.097$ & $\mathbf{6.23\times}$ \\

\zsgrp{MDS}
\zsdst{Gateway-hub}   & $\phantom{+}0.00$ & $0.103$ & $0.102$ & $1.00\times$ \\
\zsdst{Geometric-anchor} & $-0.33$           & $0.130$ & $0.104$ & $1.26\times$ \\
\zsdst{Star-kernel}      & $\phantom{+}0.00$ & $0.086$ & $0.080$ & $1.08\times$ \\
\bottomrule
\end{tabular}
\end{adjustbox}
\end{minipage}
\hfill
\begin{minipage}[t]{0.49\linewidth}
\centering
\begin{adjustbox}{max width=\linewidth}
\begin{tabular}{@{}l r @{\hspace{8pt}} rr r@{}}
\zshead

\zsgrp{MDKP}
\zsdst{Decoy-complement}     & $+0.57$           & $90.79$  & $251.35$ & $0.36\times$ \\
\zsdst{Latent-class}         & $-0.07$           & $116.04$ & $97.13$  & $1.19\times$ \\
\zsdst{Single-resource}      & $\phantom{+}0.00$ & $7.92$   & $0.22$   & $\mathbf{35.95\times}$ \\

\zsgrp{TSP}
\zsdst{Clustered Euclidean}  & $\phantom{+}0.00$ & $6.08$   & $1.60$   & $3.80\times$ \\
\zsdst{Latent-metric}        & $\phantom{+}0.00$ & $6.11$   & $4.31$   & $1.42\times$ \\
\zsdst{Paired-ribbon zigzag} & $\phantom{+}0.00$ & $0.65$   & $0.04$   & $\mathbf{17.53\times}$ \\
\bottomrule
\end{tabular}
\end{adjustbox}
\end{minipage}

\end{table}

\subsection{Synthesis Variance}
\label{app:synthesis_variance}

We repeated GPT-5.2 synthesis three times on seven representative targets. Across-run quality is relatively stable: four of the seven targets have quality standard deviation at most \(0.005\), and the mean standard deviation is \(0.018\). Runtime varies more substantially, reflecting both different synthesized implementations and, in these sweep measurements, possible machine contention. With only three runs per target, we treat this as an aggregate stability diagnostic rather than a precise estimate of per-target variance: solution quality appears more stable than the exact deployment implementation and runtime.

\subsection{Diagnostic Traces for Synthesized Solvers}
\label{app:solver_diagnostics}

We record four traces during held-out evaluation to characterize how the synthesized solvers behave at deployment. \emph{Shortcut} indicates that the solver uses its distribution-specific fast path, such as template checking in Coloring, a seeded assignment in MAXSAT, density sorting in Packing LP, restricted scoring or DP in MDKP, or structured tour construction in TSP. \emph{Fallback} indicates that the solver instead invokes its generic safety routine. \emph{Residual size} measures the subproblem remaining after specialization; because its meaning differs across problem classes, it should only be compared within a class. \emph{Repair iterations} counts the bounded local improvements made after the initial construction.

These traces describe the execution of the deployed solvers rather than causal ablations: we do not disable individual mechanisms and rerun the evaluation. We average each trace within a target and then across the three targets in each problem class.

\begin{table}[t]
\centering
\caption{
\textbf{Execution traces of the synthesized solvers on held-out instances.}
Rates and counts are averaged within each target and then across the three
targets in each problem family. The two rates need not sum to $100\%$:
remaining instances take the solver's own non-shortcut path rather than a
generic fallback. Residual size is comparable only within a family, so it is
not aggregated.
}
\label{tab:solver_diagnostics_full}

\small
\setlength{\tabcolsep}{8pt}
\renewcommand{\arraystretch}{1.15}

\begin{adjustbox}{max width=\linewidth,center}
\begin{tabular}{@{}l rr @{\hspace{16pt}} rr@{}}
\toprule
& \multicolumn{2}{c}{\textbf{Control flow} (\% of instances)}
& \multicolumn{2}{c}{\textbf{Residual work}} \\
\cmidrule(lr){2-3}\cmidrule(l){4-5}
\textbf{Problem family}
& Fast path & Fallback
& Residual size & Repair iters. \\
\midrule

Coloring    & $67.4$  & $0.0$  & $0.00$   & $0.00$ \\
MAXSAT      & $88.1$  & $7.4$  & $26.24$  & $0.00$ \\
MIS         & $66.7$  & $0.0$  & $235.92$ & $0.33$ \\
MDS         & $86.4$  & $13.6$ & $441.72$ & $0.00$ \\
Packing LP  & $100.0$ & $0.0$  & $346.78$ & $0.00$ \\
MDKP        & $100.0$ & $0.0$  & $321.03$ & $2.67$ \\
TSP         & $100.0$ & $0.0$  & $906.74$ & $0.00$ \\

\midrule
\textbf{All \(21\)}
& $\mathbf{86.9}$ & $\mathbf{3.0}$ & --- & $\mathbf{0.43}$ \\

\bottomrule
\end{tabular}
\end{adjustbox}
\end{table}

The synthesized solvers use their distribution-specific fast paths on \(86.9\%\) of test instances on average, including \(100\%\) for Packing LP, MDKP, and TSP. This is consistent with the computation patterns in Table~\ref{tab:dac_all_distributions}: density rules in Packing LP, surrogate scoring or restricted DP in MDKP, and structured tour construction in TSP are used directly on most test instances.

Fallback is uncommon, occurring on \(3.0\%\) of instances overall. It is most frequent for MDS (\(13.6\%\)) and MAXSAT (\(7.4\%\)), where the synthesized strategies combine learned structure with a generic fallback when that structure is insufficient. Repair is also limited, averaging \(0.43\) iterations, although MDKP uses more repair (\(2.67\) iterations) as part of its packing procedure. Residual size provides an additional within-family measure of how much of the original problem remains after specialization; because its meaning differs across problem classes, the values should not be compared across families.

Overall, the traces show that the synthesized solvers usually execute the specialized computation identified during synthesis rather than relying on their fallback routines. Fallback and repair remain available for instances on which the learned structure does not fully determine a solution.

\subsection{Perturbation Robustness Ablation}
\label{app:perturbation_ablation}

To test sensitivity to graph presentation, we randomly permute the vertex labels of each held-out instance. This leaves the graph unchanged up to isomorphism but changes the vertex identifiers and input order. Solvers that rely on graph structure should therefore be largely unaffected, whereas solvers that exploit labeling or ordering may change their behavior.

We evaluate each selected solver on both the original and relabeled instances. We report \(Q_{\rm orig}\), \(Q_{\rm pert}\), their difference \(\Delta Q=Q_{\rm pert}-Q_{\rm orig}\), the fraction of instances for which quality, objective value, or feasibility changes, and the runtime ratio \(t_{\rm pert}/t_{\rm orig}\). Runtime ratios are aggregated by geometric mean. Relabeling is used only for this diagnostic and plays no role in solver selection.

\begin{table}[t]
\centering
\caption{
\textbf{Aggregate graph-relabeling perturbation ablation.}
Means over the three target distributions in each graph problem, and over all
nine targets in the final row. The columns ``Qual. changed,'' ``Obj. changed,''
and ``Feas. changed'' report the fraction of instances whose quality, objective
value, or feasibility status changes under relabeling, respectively. The runtime
ratio is the geometric mean of $t_{\rm pert}/t_{\rm orig}$. Feasibility is
invariant throughout.
}
\label{tab:perturbation_ablation_aggregate}
\small
\setlength{\tabcolsep}{6pt}
\renewcommand{\arraystretch}{1.2}
\begin{adjustbox}{max width=\linewidth,center}
\begin{tabular}{@{}l c r r r r r r r@{}}
\toprule
& & \multicolumn{3}{c}{\textbf{Quality}}
& \multicolumn{3}{c}{\textbf{Change rate}}
& \multicolumn{1}{c}{\textbf{Runtime}} \\
\cmidrule(lr){3-5}
\cmidrule(lr){6-8}
\cmidrule(l){9-9}
\textbf{Problem} & \textbf{Targets}
& $Q_{\rm orig}$ & $Q_{\rm pert}$ & $\Delta Q$
& Qual. changed & Obj. changed & Feas. changed & Ratio \\
\midrule
Coloring & $3$ & $0.868$ & $0.791$ & $-0.077$ & $0.342$ & $0.342$ & $0.000$ & $1.38\times$ \\
MDS      & $3$ & $0.973$ & $0.819$ & $-0.155$ & $0.419$ & $0.333$ & $0.000$ & $1.39\times$ \\
MIS      & $3$ & $0.992$ & $0.992$ & $-0.001$ & $0.258$ & $0.191$ & $0.000$ & $1.31\times$ \\
\midrule
\textbf{All} & $9$
& $\mathbf{0.945}$ & $\mathbf{0.867}$ & $\mathbf{-0.077}$
& $\mathbf{0.340}$ & $\mathbf{0.289}$ & $\mathbf{0.000}$
& $\mathbf{1.36\times}$ \\
\bottomrule
\end{tabular}
\end{adjustbox}

\end{table}

Feasibility is unchanged by relabeling on all nine targets. Mean quality drops from \(0.945\) to \(0.867\), but almost all of the loss comes from two targets: Ring-template in Coloring and Geometric-anchor in MDS. The other seven show little change in mean quality, suggesting that most synthesized solvers rely primarily on structure that survives relabeling, while a few also exploit presentation-specific cues.

Runtime is less stable. The geometric-mean slowdown is \(1.36\times\), with Ring-template slowing by \(9.65\times\) and Core-fringe trap by \(2.19\times\). In the latter case, quality is unchanged, showing that relabeling can alter the solver's execution path even when the final solution is unaffected.

\begin{table}[t]
\centering
\caption{
\textbf{Graph-relabeling robustness by target.}
Each instance is evaluated before and after a random permutation of its vertex
labels. Runtime is reported as \(t_{\rm pert}/t_{\rm orig}\); bold marks the two
largest slowdowns.
}
\label{tab:perturbation_ablation_targets}

\small
\setlength{\tabcolsep}{5pt}
\renewcommand{\arraystretch}{1.2}

\begin{adjustbox}{max width=\linewidth,center}
\begin{tabular}{@{}llrrrrrrr@{}}
\toprule
& &
\multicolumn{3}{c}{\textbf{Quality}} &
\multicolumn{3}{c}{\textbf{Change rate}} &
\multicolumn{1}{c}{\textbf{Runtime}} \\
\cmidrule(lr){3-5}
\cmidrule(lr){6-8}
\cmidrule(l){9-9}

\textbf{Problem}
& \textbf{Target distribution}
& \(Q_{\rm orig}\)
& \(Q_{\rm pert}\)
& \(\Delta Q\)
& \textbf{Qual.\ changed}
& \textbf{Obj.\ changed}
& \textbf{Feas.\ changed}
& \textbf{Ratio} \\
\midrule

\multirow{3}{*}{Coloring}
& Ring-template
& $1.000$ & $0.772$ & $-0.228$
& $1.000$ & $1.000$ & $0.000$
& $\mathbf{9.65\times}$ \\

& Overlapping-palette
& $0.804$ & $0.801$ & $-0.004$
& $0.026$ & $0.026$ & $0.000$
& $0.28\times$ \\

& Separator-trap
& $0.800$ & $0.800$ & $\phantom{-}0.000$
& $0.000$ & $0.000$ & $0.000$
& $0.98\times$ \\

\addlinespace[3pt]

\multirow{3}{*}{MDS}
& Gateway-hub
& $0.920$ & $0.920$ & $\phantom{-}0.000$
& $0.256$ & $0.000$ & $0.000$
& $1.10\times$ \\

& Geometric-anchor
& $1.000$ & $0.536$ & $-0.464$
& $1.000$ & $1.000$ & $0.000$
& $2.18\times$ \\

& Star-kernel
& $1.000$ & $1.000$ & $\phantom{-}0.000$
& $0.000$ & $0.000$ & $0.000$
& $1.11\times$ \\

\addlinespace[3pt]

\multirow{3}{*}{MIS}
& Clique-path
& $0.991$ & $0.991$ & $-0.001$
& $0.386$ & $0.350$ & $0.000$
& $1.01\times$ \\

& Core-fringe trap
& $1.000$ & $1.000$ & $\phantom{-}0.000$
& $0.000$ & $0.000$ & $0.000$
& $\mathbf{2.19\times}$ \\

& Motif-bridge
& $0.986$ & $0.985$ & $-0.001$
& $0.388$ & $0.224$ & $0.000$
& $1.02\times$ \\

\bottomrule
\end{tabular}
\end{adjustbox}

\end{table}

Overall, relabeling leaves feasibility unchanged and has little effect on quality for most targets. The two exceptions, Ring-template and Geometric-anchor, show that some synthesized solvers also exploit vertex labeling or input order rather than relying entirely on isomorphism-invariant graph structure.

\end{document}